\documentclass[12pt]{article}
\usepackage[left = 1.25in, right = 1.25in]{geometry}

\usepackage{natbib}
\usepackage{amscd,amssymb,amsmath,mathrsfs,amsfonts,graphicx,verbatim,epsfig,
psfrag, float, fancybox} 
\usepackage{graphics}                 
\usepackage{color}                    
\usepackage{hyperref}                 
\usepackage[TS1,OT1,T1]{fontenc}

\usepackage{amsmath}    
\usepackage{amscd,amssymb,mathrsfs,amsfonts}
\usepackage{graphicx}   
\usepackage{verbatim}   
\usepackage{color}      
\usepackage{amsmath}
\usepackage{verbatim}
\usepackage{setspace} 
\usepackage{algorithm}

\usepackage{authblk}

\newcommand{\Keywords}[1]{\par\noindent
{\small{\em Keywords\/}: #1}}

\newtheorem{theorem}{Theorem}[section]
\newtheorem{lemma}[theorem]{Lemma}

\newtheorem{proof}[theorem]{Proof}

\begin{document}

\title{A Permutation Approach to Testing Interactions in Many Dimensions}

\author{Noah Simon \thanks{Department of Statistics, Stanford University, \texttt{nsimon@stanford.edu}} \and Rob Tibshirani\thanks{Department of Statistics, Stanford University, Department of Health Research and Policy, Stanford University}}

\maketitle

\begin{abstract}
To date, testing interactions in high dimensions has been a challenging task. Existing methods often have issues with sensitivity to modeling assumptions and heavily asymptotic nominal p-values. To help alleviate these issues, we propose a permutation-based method for testing marginal interactions with a binary response. Our method searches for pairwise correlations which differ between classes. In this manuscript, we compare our method on real and simulated data to the standard approach of running many pairwise logistic models. On simulated data our method finds more significant interactions at a lower false discovery rate (especially in the presence of main effects). On real genomic data, although there is no gold standard, our method finds apparent signal and tells a believable story, while logistic regression does not. We also give asymptotic consistency results under not too restrictive assumptions.\\

\end{abstract}

\section{Introduction}\label{sec:intro}
In many areas of modern science, massive amounts of data are generated.
In the biomedical sciences,  examples arise in
genomics, proteomics,  and flow cytometry.  New high-throughput experiments allow
researchers to look at the dynamics of very rich systems.  With these vast increases in data accumulation,
scientists have found classical statistical techniques in need of improvement, and classical notions of error control (type 1 error)
overwhelmed.

Consider the following two class situation: our data consists of $n$
observations, each observation with a known class label of 1 or 2,
 with $p$ covariates measured per observation. Let $y$ denote
the $n$-vector corresponding to class (with $n_1$ observations in
class $1$ and $n_2$ in class $2$),
and $X$, the $n\times p$ matrix of covariates. We often assume each row
of $X$ is independently normally distributed with some class specific mean
$\mu_{y(i)}\in\mathbb{R}^p$ and covariance $\Sigma_{y(i)}$ (for
instance in quadratic discriminant analysis). Here, we are interested in differences between classes. A
common example is gene expression data on healthy and diseased
patients: the covariates are the genes ($p\sim 20,000$), the
observations are patients ($n\sim 100$) belonging to either
the healthy or diseased class. Here, one might look at differences
between classes to develop a genetic prognostic test of the disease,
or to better understand its underlying biology.
Recent high dimensional procedures have
focused on detecting differences between
$\mu_1$ and $\mu_2$ by considering them one
covariate at a time.

In this paper we consider the more difficult problem of
testing marginal interactions. In a fashion similar to the approaches used in large scale
testing of main effects (see e.g \citet{DSB2003}, \citet{TTC01} and
\citet{efron2010ebayes}), we do this on a pair by pair basis.

The standard approach for this problem has been to run many bivariate logistic regressions and then conduct a post-hoc analysis on the nominal p-values. \citet{buzkova2011} has a nice summary of the subtle  issues that arise in testing
for just a single interaction in a regression framework.
In  particular, a permutation approach cannot be simply applied because
it tests the null hypothesis of both no interaction and no main effects at the same time. In the high-dimensional setting with FDR estimates, these issues are compounded.

The logistic regression based methods are all derived from what we call a {\em forward model}, that is, a  model 
for the conditional distribution of $Y|X$. In contrast, a {\em backward model} (discussed below) is a  model 
for the conditional distribution of $X|Y$. We propose a method, based on a backwards model, to approach this same problem. By using this backwards framework we avoid many of the pitfalls of standard approaches: we have a less model-based method, we attack a potentially more scientifically interesting quantity, and we can use a permutation null for FDR estimates. Our approach is unfortunately only for binary response --- the backwards model is more difficult to work with for continuous $y$.

In this paper we develop our method, and show its efficacy as compared to straightforward logistic regression on real and simulated data. We explain how to deal with nuisance variables, and give insight into our permutation-based estimates of FDR. We also give some asymptotic consistency results. 

\section{Existing Methods}\label{sec:exist}

We begin by going more in-depth on the standard approach and its issues. In general one might like to specify a generative logistic model for the data (a forward model) of the form
\begin{equation}\label{eq:inter}
\operatorname{logit}\left[\operatorname{P}(y_i = 1 | X_{i,\cdot})\right] = \beta_0 + \sum_{j=1}^p \beta_j X_{i,j} + \sum_{k< j}\gamma_{j,k}
X_{i,j} X_{i,k}
\end{equation}
where $X_{i,\cdot}$ is the $i$-th row of $X$, and test if the $\gamma_{j,k}$ are nonzero in this model. 
Here $i$ indexes the observations and $j,k$ index the  predictors.
However, because it is a joint rather than a marginal model, this does not easily allow us to test individual pairs of covariates
separately from the others. Furthermore in the scenario with $n < p(p+1)/2$, the MLE for this model is not well defined (one can always get perfect separation) and non-MLE estimates are very difficult to use for testing.

Alternatively, for each pair $(X_{i,j}, X_{i,k})$  one might assume a generative logistic model of the form
\begin{equation}\label{eq:bivLog}
\operatorname{logit}\left[\operatorname{P}(y_i = 1 | X_{i,j},
  X_{i,k})\right] = \beta_0 + \beta_j X_{i,j} +  \beta_k X_{i,k} + \gamma_{j,k}X_{i,j} X_{i,k}
\end{equation}
and estimate or test $\gamma_{j,k}$ using the MLE $\hat\gamma_{j,k}$.

A standard approach to this problem in the past has been to fit pairwise logistic models~\eqref{eq:bivLog} independently for every pair $(j,k)$, and then use standard tools (ie. asymptotic normality of the MLE) to calculate approximate $P$-values. Once the $p(p-1)/2$ $p$-values are calculated, the approach of \citet{BH95} or some other standard procedure can be used to estimate/control FDR.

This approach has a number of problems. First of all, while the approach is very model-based, one cannot even ensure that all of the bivariate logistic models are consistent with one another (i.e. that there is a multivariate model with the given marginals). In particular, model misspecification will often cause over-dispersion resulting in anti-conservative FDR estimates. Also, if the true model contained quadratic terms (which we do not have in our model) then for correlated pairs of features this approach will compensate by trying to add false interactions. Even if we did believe the model, the p-values are only approximate, and this approximation grows worse as we move into the tails. 

One might hope to avoid some of these issues by using permutation p-values, however, as shown in \citet{buzkova2011} permutation methods are incongruous with this approach --- they test the joint null hypothesis of no main effect or interaction, which is not the hypothesis of interest. This difficulty is also discussed in \citet{pesarin2001}. In an attempt to resolve this, \citet{kooperberg2008} regress out the main effects before permuting the residuals. This is a nice adjustment, but is still heavily model-based.

To deal with these issues, we take a step back and use a different generative model. Our generative model has an equivalent logistic model and this correspondence allows us to sidestep many of the issues with the standard logistic approach.

\subsection{Forward vs Backward Model}\label{sec:forVsback}
We propose to begin with a ``backward'' generative model
--- as mentioned in Section~\ref{sec:intro}, we assume that observations are Gaussian in each class $\left(x_i|y_i\right) \sim
N(\mu_{y(i)}, \Sigma_{y(i)})$ with a class specific mean and covariance matrix. We argue that the most natural test of interaction is a test of equality of correlations between groups.

Toward this end, let us apply Bayes theorem to our backwards generative model, to obtain
\begin{align*}
\operatorname{P}(y = 1 | x) &= \frac{\pi_1 \operatorname{exp}\left(l_1\right)}{\pi_2  \operatorname{exp}\left(l_2\right) + \pi_1
   \operatorname{exp}\left(l_1\right)}\\
&= \frac{ \operatorname{exp}\left[\operatorname{log}(\pi_1/\pi_2) +
    l_1 - l_2\right]}{1 +  \operatorname{exp}\left[\operatorname{log}(\pi_1/\pi_2) +
    l_1 - l_2\right]}
\end{align*}
where
\[
l_m = -p\operatorname{log}\left(2\pi\right)/2 - \operatorname{logdet}\left(\Sigma_m\right)/2 - (x-\mu_m)^{\top}\Sigma_m^{-1}(x-\mu_m)/2
\]
and $\pi_m$ is the overall prevalence of class $m$. We can simplify this to
\begin{align*}
\operatorname{logit} \left(P\right) &= \operatorname{logdet}\left(\Sigma_2\right)/2 - \operatorname{logdet}\left(\Sigma_1\right)/2 + \operatorname{log}(\pi_1/\pi_2) + \mu_2^{\top}\Sigma_2^{-1}\mu_2/2\\
&- \mu_1^{\top}\Sigma_1^{-1}\mu_1/2 + \left(\Sigma_1^{-1}\mu_1 - \Sigma_2^{-1}\mu_2\right)^{\top} x + x^{\top}\left(\Sigma_2^{-1} - \Sigma_1^{-1}\right) x/2.
\end{align*}
This is just a logistic model with interactions and quadratic terms, and in the form of \eqref{eq:inter} (with additional quadratic terms) we have
\begin{align*}
\beta_0 &= \operatorname{logdet}\left(\Sigma_2\right)/2 - \operatorname{logdet}\left(\Sigma_1\right)/2 + \operatorname{log}(\pi_1/\pi_2)\\
 &+ \mu_2^{\top}\Sigma_2^{-1}\mu_2/2 - \mu_1^{\top}\Sigma_1^{-1}\mu_1/2\\
\beta_{j} &= \left(\Sigma_1^{-1}\mu_1 - \Sigma_2^{-1}\mu_2\right)_j\\
\gamma_{j,k} &= \left(\Sigma_2^{-1} - \Sigma_1^{-1}\right)_{j,k}.
\end{align*}
From here we can see that traditional logistic regression interactions in the full model correspond to nonzero off-diagonal elements of $\Sigma_2^{-1} - \Sigma_1^{-1}$. Testing for non-zero elements here is not particularly satisfying for a number of reasons. Because coordinate estimates are so intertwined, there is no simple way to marginally test for non-zero elements in $\Sigma_2^{-1} - \Sigma_1^{-1}$ --- in particular there is no straightforward permutation test. Also, for $n<p$ the MLEs for the precision matrices are not well defined.

As in the logistic model~\eqref{eq:bivLog} we may condition on only a pair of covariates $j$ and $k$ in our backwards model. Using Bayes theorem as above, our equivalent bivariate forward model is
\begin{align*}
\operatorname{P}(y = 1 |\, \tilde{x} = \left(x_j, x_k\right)^{\top}) &= \operatorname{log}(\pi_1/\pi_2) + \tilde{\mu}_2^{\top}\tilde{\Sigma}_2^{-1}\tilde{\mu}_2/2 -
\tilde{\mu}_1^{\top}\tilde{\Sigma}_1^{-1}\tilde{\mu}_1/2\\
& + \left(\tilde{\Sigma}_1^{-1}\tilde{\mu}_1 - \tilde{\Sigma}_2^{-1}\tilde{\mu}_2\right)^{\top} \tilde{x} + \tilde{x}^{\top}\left(\tilde{\Sigma}_2^{-1} - \tilde{\Sigma}_1^{-1}\right) \tilde{x}/2
\end{align*}
where $\tilde{\mu}_m$ and $\tilde{\Sigma}_m$ are the mean vector and covariance matrix in class $m$ for only $X_j$ and $X_k$. Hence the backwards model has an equivalent logistic model similar to ~\eqref{eq:bivLog} but with quadratic terms included as well. One should note that the main effect and interaction coefficients in this marginal model \emph{do not} match those from the full model (i.e. the marginal interactions and conditional interactions are different).

Our usual marginal logistic interaction between covariates $j$ and $k$ corresponds to a nonzero off-diagonal entry in $\tilde{\Sigma}_2^{-1} - \tilde{\Sigma}_1^{-1}$. Simple algebra gives
\[
\tilde{\Sigma}^{-1}_{m(1,2)} = -\left(\frac{R_{m(j,k)}}{\sigma_{m(j)}\sigma_{m(k)}\left(1-R_{m(j,k)}^2\right)}\right)
\]
where $R_{m(j,k)}$ is the correlation between features $j$ and $k$ in class $m$, and $\sigma_{m(j)}$ is the standard deviation of variable $j$ in class $m$.

Thus, if we were to test for ``logistic interactions'' in our pairwise backwards model, we would be testing:
\[
\frac{R_{1(j,k)}}{\sigma_{1(j)}\sigma_{1(k)}\left(1-R_{1(j,k)}^2\right)} = \frac{R_{2(j,k)}}{\sigma_{2(j)}\sigma_{2(k)}\left(1-R_{2(j,k)}^2\right)}
\]
Now, if $\sigma_{1(j)} = \sigma_{2(j)}$, and $\sigma_{1(k)} = \sigma_{2(k)}$, then this is equivalent to testing if $R_{1(j,k)} = R_{2(j,k)}$. If not, then a number of unsatisfying things may happen. For example if the variance of a single variable changes between classes, then, even if its correlation with other variables remains the same, it still has an ``interaction'' with all variables with which it is correlated. This change of variance is a characteristic of a single variable, and it seems scientifically misleading to call this as an ``interaction'' between a pair of features.

Toward this end, we consider a restricted set of null hypotheses --- rather than testing for an interaction between each pair of features $(j,k)$, we test the null $R_{1(j,k)} = R_{2(j,k)}$. Not all logistic interactions will have $R_{1(j,k)} \neq R_{2(j,k)}$, but we believe this is the property which makes an interaction physically/scientifically interesting.

To summarize, there are a number of issues in the forward model which are alleviated through the use of the backwards model:
\begin{itemize}
\item The marginal forward models are not necessarily consistent (one cannot always find a ``full forward model'' with the given marginals).
\item Omitted quadratic terms may be mistaken for interactions between correlated covariates.
\item Interesting interactions are only those for which $R_{1(j,k)} \neq R_{2(j,k)}$.
\item $P$-values are approximate and based on parametric assumptions.
\end{itemize}

\section{Proposal}\label{sec:method}
We begin with the generative model described in
Section~\ref{sec:forVsback}--- we assume observations are Gaussian in each class $\left(x_i|y_i\right) \sim
N(\mu_{y(i)}, \Sigma_{y(i)})$ with a class specific mean and covariance matrix. As argued above, we test for interactions by testing
\[
\mathbf{H}_{j,k}:\,R_{1(j,k)} = R_{2(j,k)}
\]
for each $j<k$, where again, $R_{m(j,k)}$ denotes the $(j,k)$-th entry of the correlation matrix for class $m$.

If we were only testing one pair of covariates $(j,k)$, a
straightforward approach would be to compare the sample correlation coefficients $\hat{R}_{1(j,k)}$ to
$\hat{R}_{2(j,k)}$. In general, because the variance of $\hat{R}_{m(j,k)}$ is dependent on
$R_{m(j,k)}$, it is better to make inference on a Fisher
transformed version of $\hat{R}_{m(j,k)}$:
\[
U_{m(j,k)} = \operatorname{arctanh}\left(\hat{R}_{m(j,k)}\right)
\dot{\sim}
N\left(\operatorname{arctanh}\left(R_{m(j,k)}\right),\frac{1}{n_m-3}\right)
\]
This is a variance stabilizing transformation. Now, to compare the two
correlations we consider the statistic
\begin{equation}\label{eq:stat}
T_{(j,k)} = U_{1(j,k)} - U_{2(j,k)} \dot{\sim}
N\left(\operatorname{arctanh}\left(R_{1(j,k)}\right) - \operatorname{arctanh}\left(R_{2(j,k)}\right),\frac{1}{n_1-3} + \frac{1}{n_2-3}\right)
\end{equation}
Under the null hypothesis: $R_{1(j,k)} = R_{2(j,k)}$, this statistic is
distributed $N\left(0,\frac{1}{n_1-3} + \frac{1}{n_2-3}\right)$. To test if the
correlations are equal we need only compare our statistic $T_{(j,k)}$
to its null distribution and find a $p$-value. While this approach works well for single tests, because we are in the high dimensional setting we use a different approach which doesn't rely on the statistic's asymptotic normal distribution.

We are interested in testing differences between two large correlation matrices in higher dimensional spaces. We again calculate the differences of our transformed sample correlations
--- we now calculate $p(p-1)/2$ statistics; one for each pair $(j,k)$
with $j<k$. However to assess significance we no longer just compare each statistic
to the theoretical null distribution and find a p-value. Instead we directly estimate false discovery rates (FDR):  we choose some
threshold for our statistics, $t$, and reject (/call significant)
all $(j,k)$ with $|T_{(j,k)}| > t$. Clearly, not all marginal interactions called significant
in this way will be truly non-null and it is important to estimate the FDR
of the procedure for this cutoff, that is
\[
\operatorname{FDR} = E\left[\frac{\textrm{\# false rejections}}{\textrm{\# total rejections}}\right],
\]
where `\#' is short-hand for ``number of''. It is standard to approximate this quantity by 
\begin{equation}\label{eq:FDR}
\frac{\hat{E}[\textrm{\# false rejections}]}{\textrm{\# total rejections}}.
\end{equation}
The denominator is just the number of $|T_{(j,k)}| > t$ (which we know). If we knew
which hypotheses were null and their distributions then one could find
the numerator by
\begin{equation}\label{eq:numer}
E[\textrm{\# false rejections}] = \sum_{(j,k) \textrm{ null}} \operatorname{P}(|T_{(j,k)}| > t)
\end{equation}
Clearly we don't know which hypotheses are null.  To estimate
\eqref{eq:numer} we propose the following permutation approach.

We first center and scale our variables within class: for each observation we subtract off the class mean for each feature and divide by that feature's within-class standard deviation --- let $\tilde{X}$ denote this standardized matrix. This standardization doesn't change our original statistics, $T_{j,k}$ (the
correlation calculated from $X$ and $\tilde{X}$ are identical), but
is important for our null distribution. Now, let $\pi$
be some random permutation of $\{1,\ldots,n\}$. Thus, $\pi(y)$ is a
random permutation of the class memberships of the standardized variables (we keep the standardization from before the permutation). With these new class
labels we calculate a new set of $p(p-1)/2$ statistics,
$\{T^{*a}_{(j,k)}\}_{j<k}$. We can permute our data $A$ times, and
gather a large collection of these null statistics, ($Ap(p-1)/2$) of them.  To estimate
$E[\textrm{\# false rejections}]$, we take the average number
of these statistics that lie above our cutoff
\[
\hat{E}[\textrm{\# false rejections}] = \frac{1}{A} \sum_{a=1}^A \#
\{|T^{*a}_{(j,k)}| > t\}
\]
Often, one is interested in the FDR of the $l$ most significant
interactions.  In this case the cutoff, $t$, is chosen to be the
absolute value of the $l$-th most significant statistic, denoted
$T(l)$. We refer to this procedure as Testing Marginal Interactions through correlation (TMIcor) and summarize it below.

\medskip
\begin{center}
{\bf TMIcor: Algorithm for Testing Marginal Interactions}\label{alg:1}
\end{center}
\begin{enumerate}
\item Mean center and scale $X$ within each group.
\item Calculate the  feature correlation matrices $\hat R_1$ and $\hat R_2$ within each class.
\item Fisher transform the entries (for $j<k$): $U_{m(j,k)} = \operatorname{arctanh}\left(\hat{R}_{m(j,k)}\right)$\\
and take their coordinate-wise differences: $T_{(j,k)} = U_{1(j,k)} - U_{2(j,k)}$
\item for $a=1,\ldots\, ,A$ execute the following
\begin{enumerate}
\item Randomly permute class labels of the standardized variables.
\item Using the new class labels, reapply steps 2-4 to calculate new
  statistics $\{T^{*a}_{(j,k)}\}_{j<k}$
\end{enumerate}
\item Estimate FDR for any $l$ most significant interactions by
\[
\widehat{{\rm FDR}} = \frac{\left(\frac{1}{A}\right) \sum_{a=1}^A \#\{|T^{*a}_{(j,k)}| > T(l)\}}{l}
\]
\end{enumerate}
Using this approach, one gets a ranking of pairs of features and an FDR estimate for every position in the ranking. Furthermore, rather than testing for interactions between all pairs
of variables, one may instead test for interactions between variables in one set
(such as genes) and variables in another (such as environmental variables). To do this, one would only need restrict the statistics considered in steps $3$, $4b$ and $5$.

Standardizing in step $(1)$ before permuting may seem strange, but in this case is necessary. If we do not standardize first, we are testing the joint null that the means, variances and correlations are the same between classes. This is precisely what we moved to the backward model to avoid --- by standardizing we avoid permuting the ``main effects''. We discuss this permutation-based estimate of FDR in more depth in appendix A.

\section{Comparisons}\label{sec:comparisons}

In this section we apply TMIcor and the standard logistic approach to real and simulated data. On simulated data we see that in some scenarios (in particular with main effects) the usual approach has serious power issues as compared to TMIcor. Similarly on our real dataset we see that the usual approach does a poor job of finding interesting interactions, while TMIcor does well.

\subsection{Simulated Data}

We attempt to simulate a simplified version of biological data. In general, groups of proteins or genes act in concert based on biological processes. We model this with a block diagonal correlation matrix --- each block of proteins/genes is equi-correlated. This can be interpreted as a latent factor model --- all the proteins in a single block are highly correlated with the same latent variable (maybe some unmeasured cytokine), and conditional on this latent variable, the proteins are all uncorrelated. In our simulations we use $10$ blocks, each with $10$ proteins ($100$ total proteins). We simulate the proteins for our healthy controls as jointly Gaussian with $0$ mean and covariance matrix
\[
\Sigma_1 = \begin{pmatrix} 
R_1 & 0 & \cdots & 0\\
0 & R_2 & \cdots & 0\\
\vdots & \vdots & \vdots & \vdots\\
0 & \cdots & 0 & R_{10}
\end{pmatrix}
\]
where each $R_i$ is a $10\times 10$ matrix with $1$s along the diagonal, and a fixed $\rho_i>0$ for all off-diagonal entries. Now, for our diseased patients we again use mean $0$ proteins, but change our covariance matrix to
\[
\Sigma_2 = \begin{pmatrix} 
\tilde{R}_1 & 0 & \cdots & 0\\
0 & R_2 & \cdots & 0\\
\vdots & \vdots & \vdots & \vdots\\
0 & \cdots & 0 & R_{10}
\end{pmatrix}
\]
where $\tilde{R}_1$ has $1$s on the diagonal and $\tilde{\rho}_1$ for all off-diagonal entries (with $0\leq \tilde{\rho}_1 \neq \rho_1$). This correlation structure would be indicative of a mutation in the cytokine for the first group causing a change in the association between that signaling protein and the rest of the group.

Within each class (diseased and healthy) we simulated $250$ patients and applied TMIcor and the usual logistic approach. We averaged the true and estimated false discovery rates of these methods over $10$ trials. As we can see from Figure~\ref{fig:1} TMIcor outperforms the logistic approach. This difference is particularly pronounced in the second plot of Figure~\ref{fig:1}. In this plot, because the correlations are large but different in both groups ($\rho_1 = 0.3$, $\tilde{\rho}_1 = 0.6$), there are some moderate quadratic effects in the true model --- this induces a bias in the logistic approach and its FDR suffers. In contrast, these quadratic effects are not problematic in the backward framework.

  \begin{figure}[t!]
  \centerline{
    \mbox{\includegraphics[width=2.85in]{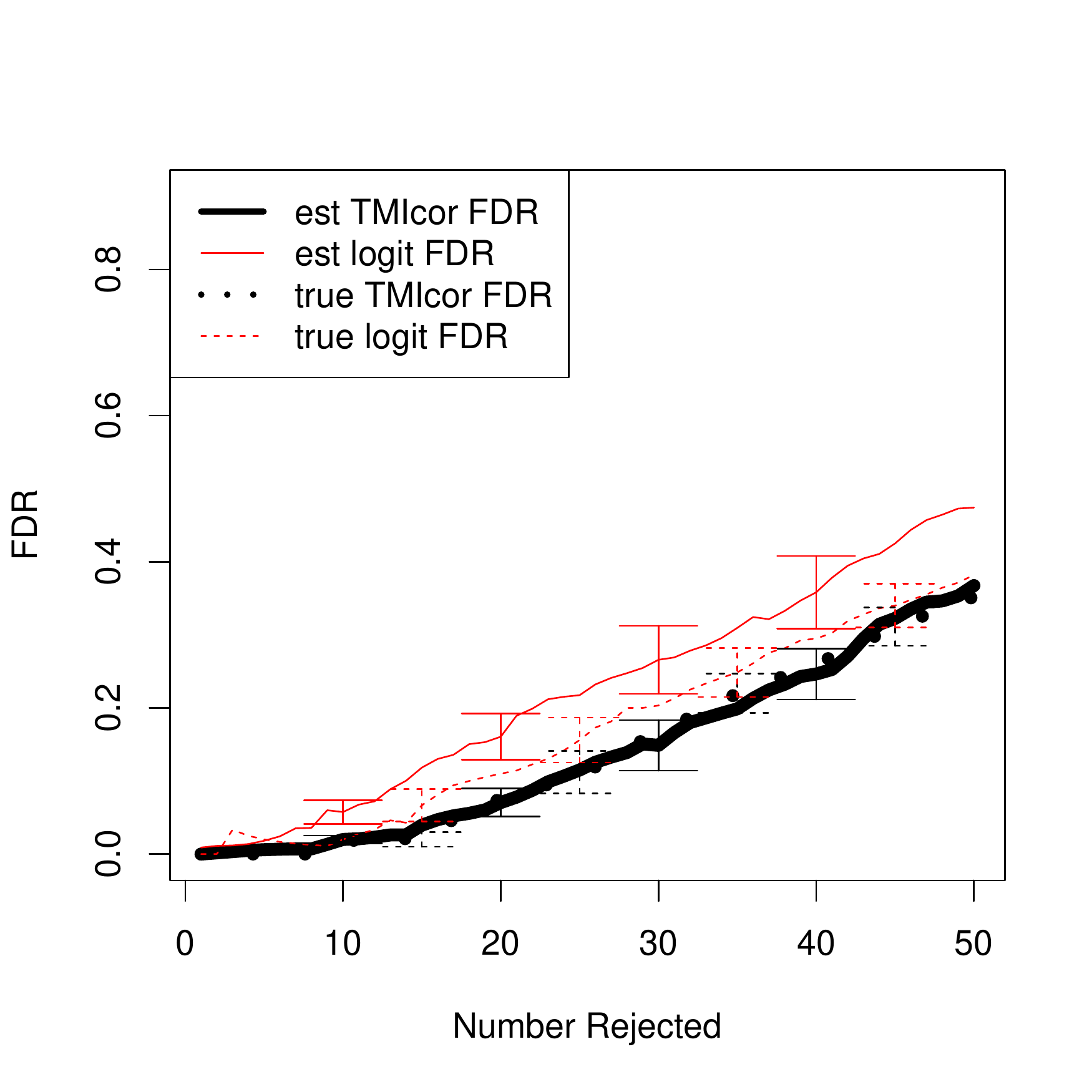}}
    \mbox{\includegraphics[width=2.85in]{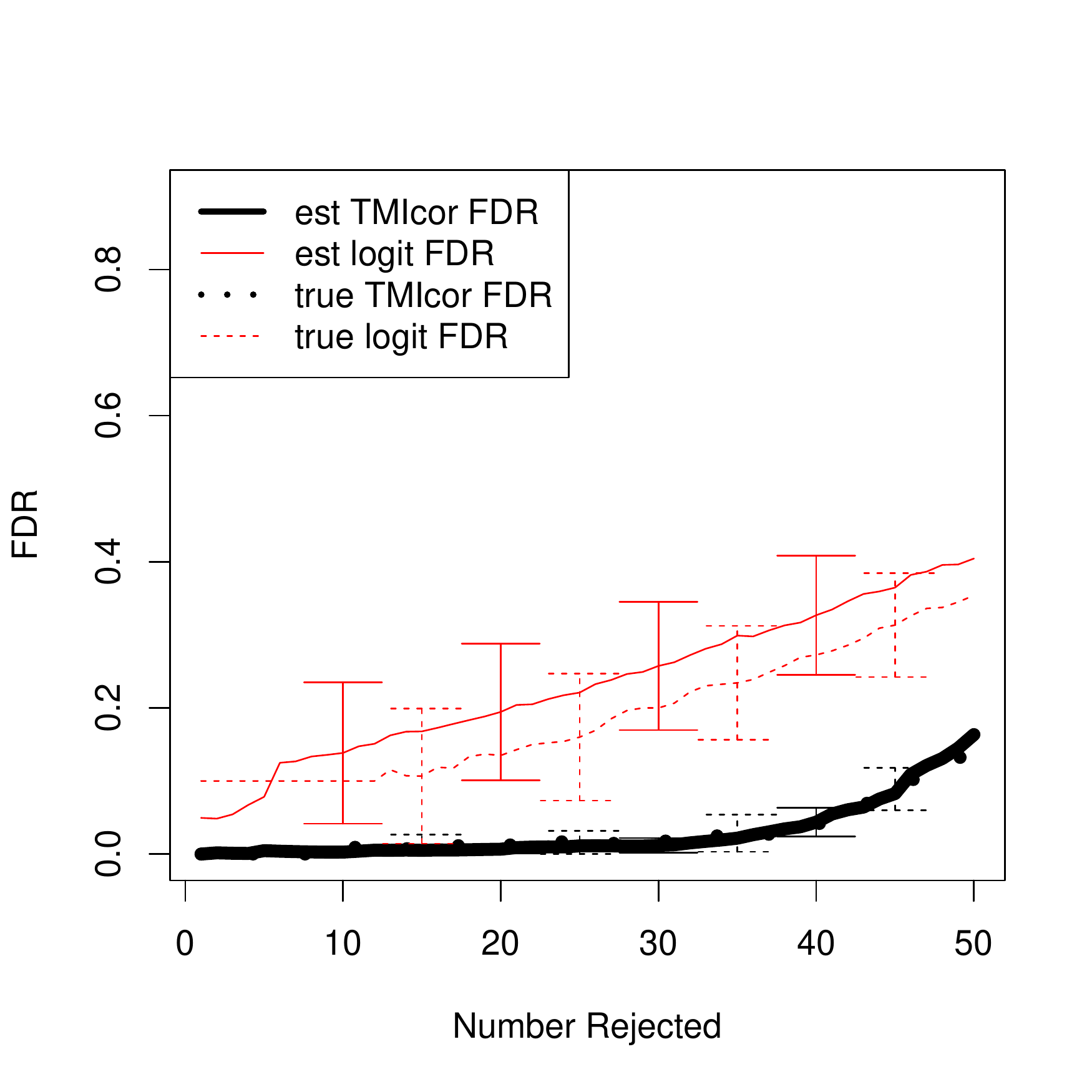}}
  }
  \caption{Plots of estimated and true FDR for TMIcor and logistic regression averaged over $10$ trials. Error bars contain the mean value $\pm$ 1 se of the mean. For controls, $\rho_i = 0.3$ for all $i$. On the left $\tilde{\rho}_1 = 0$, while on the right $\tilde{\rho}_1 = 0.6$. There is no main effect in either panel.}
  \label{fig:1}
  \end{figure}

We also consider a second set of simulations. This set used $\rho_i = 0.3$ for all $i$ and $\tilde{\rho}_1 = 0$. However, instead of mean $0$ in both classes, we set the mean for all proteins in block 1 for diseased patients to be some $\tilde{\mu}_1$ ($> 0$). The results are plotted in Figure~\ref{fig:2}. This mean shift had no effect on TMIcor (the procedure is meanshift invariant), but as the mean difference grows, it becomes increasingly difficult for the logistic regression to find any interactions. This issue is especially important as, biologically, one might expect that genes with main effects to be more likely to have true marginal interactions (and these interactions may also be more scientifically interesting).

  \begin{figure}[t!]
  \centerline{
    \mbox{\includegraphics[width=2.85in]{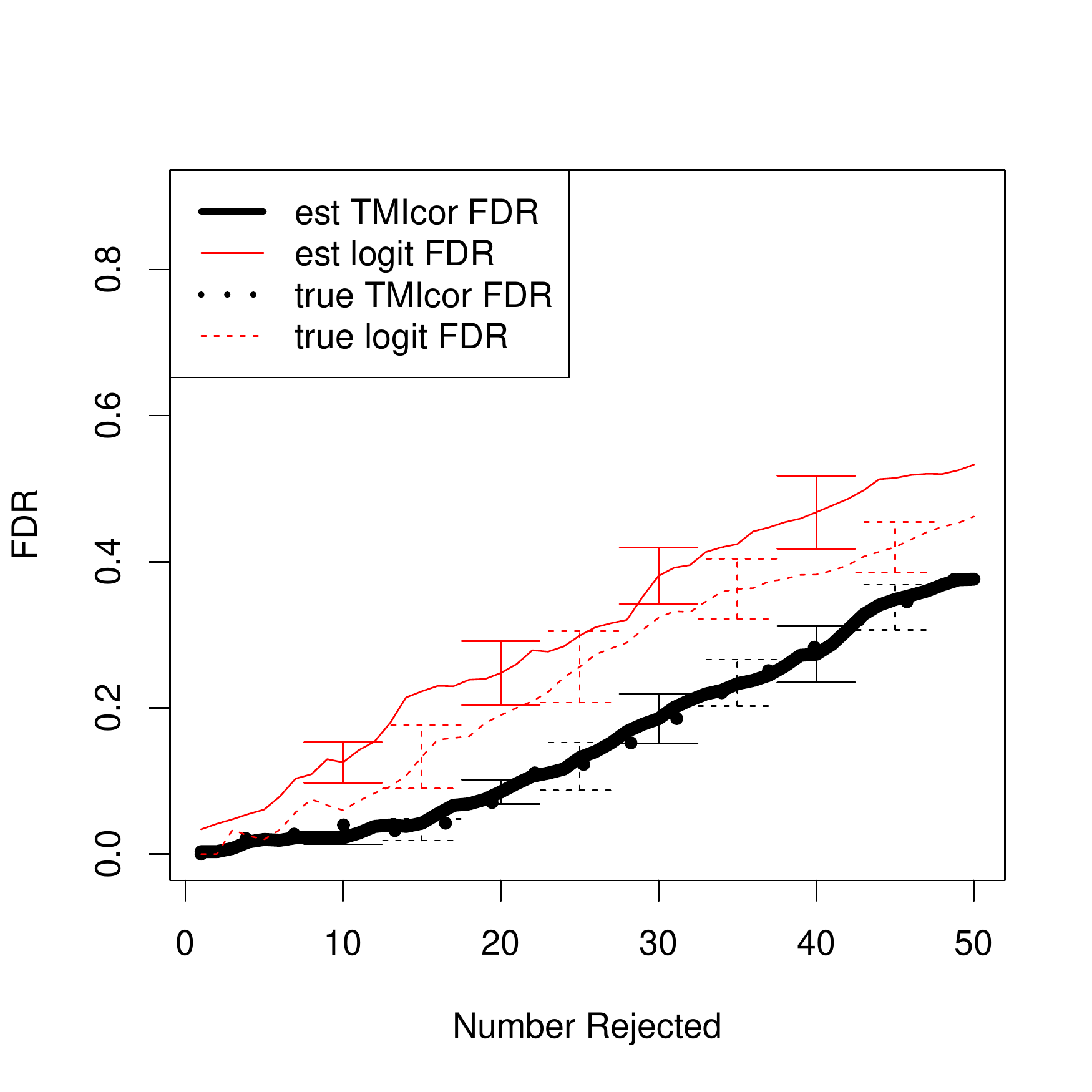}}
    \mbox{\includegraphics[width=2.85in]{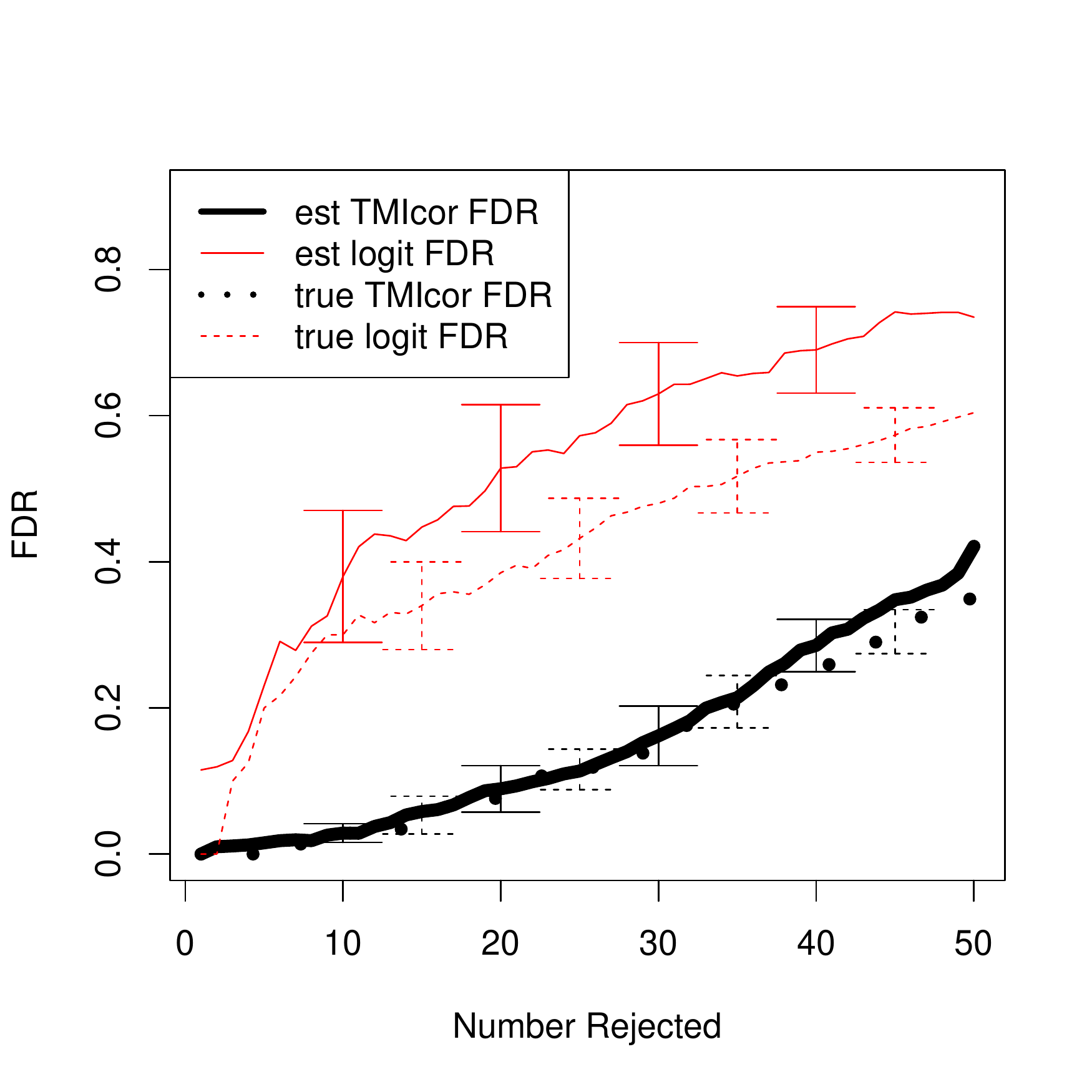}}
  }
  \caption{Plots of estimated and true FDR for TMIcor and logistic regression averaged over $10$ trials. Error bars contain the mean value $\pm$ 1 se of the mean. For both plots $\tilde{\rho}_1 = 0$ and $\rho_i = 0.3$ for all $i$. Both panels have main effects --- on the left $\tilde{\mu}_1 - \mu_1=0.5$, while on the right $\tilde{\mu}_1 - \mu_1 = 1$.}
  \label{fig:2}
  \end{figure}

While these simulations are not exhaustive, they give an indication of a number of scenarios in which TMIcor significantly outperforms logistic regression. More exhaustive simulations were run and the results mirrored those in this section.

\subsection{Real Data}

We also applied both TMIcor and logistic regression to the colitis gene expression data of \citet{burczynski2006}. In this dataset, there are $127$ total patients, $85$ with colitis ($59$ Crohn's patients + $26$ ulcerative colitis patients) and $42$ healthy controls. We restricted our analysis to the $101$ patients without ulcerative colitis. Each patient had expression data for $22283$ genes run on an Affymetrix U133A microarray. Because chromosomes $5$ and $10$ have been indicated in Crohn's disease, we enriched our dataset by using only the genes on these chromosomes, along with the $NOD2$ and $ATG16L1$ genes (chromosomes as specified by the $C1$ geneset from \citet{subramanian2005}). In total $663$ genes were used. Some of these genes were measured by multiple probesets --- the final expression values used for those genes were the average of all probesets.

From these $663$ genes we have $219,453$ of interactions to consider. Figure~\ref{fig:fdr} shows the estimated FDR curves for the two methods. TMIcor finds many more significant interactions --- at an FDR cutoff of $0.1$, TMIcor finds $2570$ significant interactions, while the logistic approach finds $15$. The significant $15$ from the logistic approach may not even be entirely believeable --- the smallest p-value of the $15$ is roughly $1/219453$, which is what we would expect it to be if all null hypotheses were true. Because the smallest p-value is large, we see that the FDR for logistic regression begins surprisingly high. The FDR subsequently drops because there are a number of p-values near the smallest, however, the significance of these hypotheses is still suspect.

\begin{figure}[!t]
  \centerline{
    \includegraphics[width=3in]{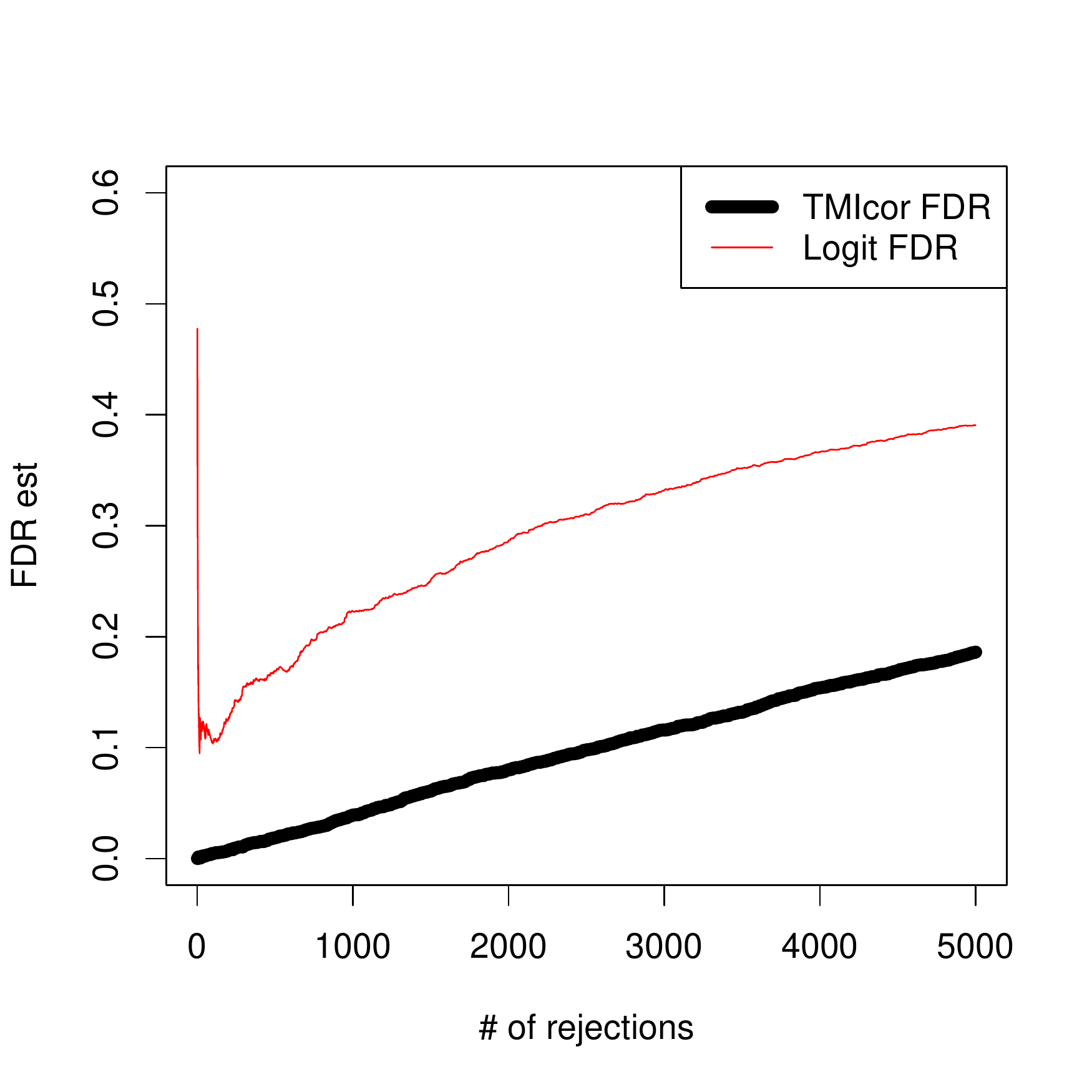}
  }
  \caption{Corhn's data; FDR estimates for TMIcor and logistic approaches for the $5000$ most significant marginal interactions}
\label{fig:fdr}
\end{figure}

Unfortunately interpreting $2570$ marginal interactions is difficult (even if all are true). Toward this end we consider the graphical representation of our analysis in Figure~\ref{fig:graphBig}. Each gene is a node in our graph, and edges between genes signify marginal interactions. In this plot we considered only the $1250$ of the $2570$ significant marginal interactions indicative of a decrease in correlation from healthy control to Crohn's (ie. $T_{j,k} > 0$). There is one large connected component, a few connected pairs and a large number of isolated genes. The connected component appears to be split into $2$ clusters. To get a better handle on this, we considered a more stringent cutoff for significant interactions --- at an FDR cutoff of $0.03$, we are left with $832$ significant interactions of which only $402$ have $T_{j,k} > 0$. We plot this graph in Figure~\ref{fig:graphSmall}: we see that our large connected component has divided into $2$. From here we further zoomed in on each component (now displaying only the $50$ most significant interactions per component), and can actually see which genes are are most important (in figure~\ref{fig:graphComp}).\\
\begin{figure}[!t]
  \centerline{
    \includegraphics[width=3in]{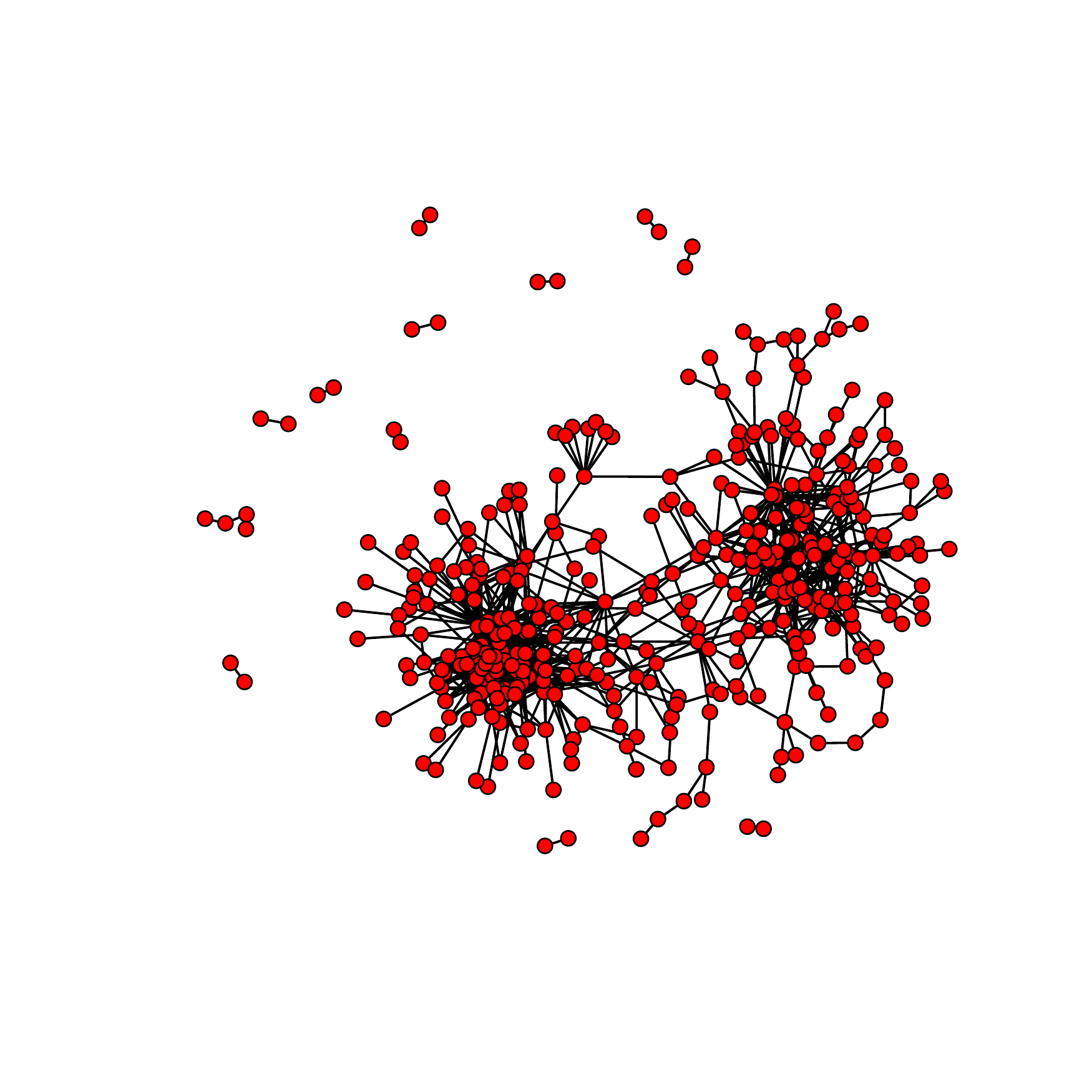}
  }
  \caption{Graph of $1250$ marginal interactions (with decreasing correlation) significant at FDR cutoff of $0.1$. Genes with no significant interactions not shown}
\label{fig:graphBig}
\end{figure}

\begin{figure}[!t]
  \centerline{
    \includegraphics[width=3in]{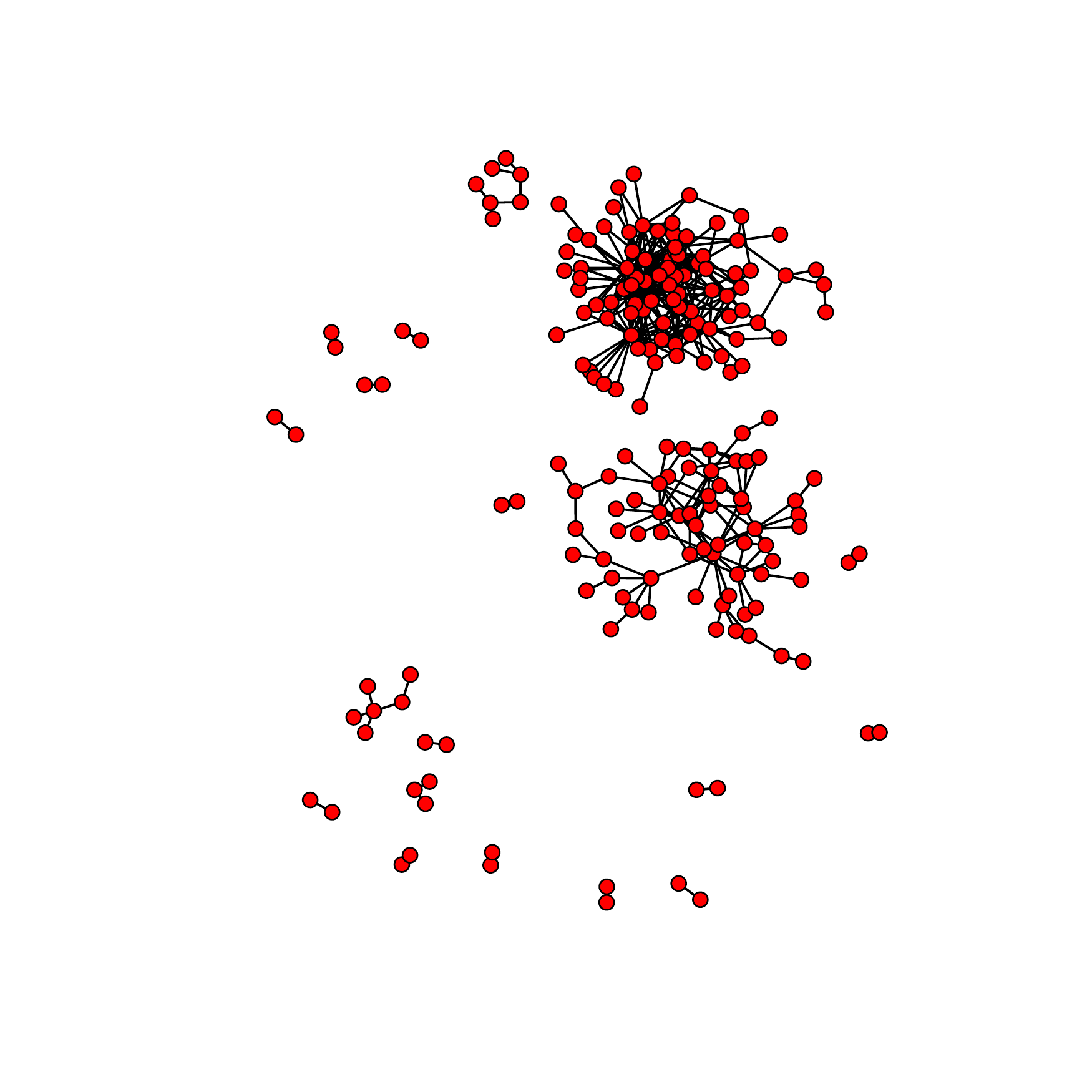}
  }
  \caption{Graph of $402$ marginal interactions (with decreasing correlation) significant at FDR cutoff of $0.03$. Genes with no significant interactions not shown}
\label{fig:graphSmall}
\end{figure}

  \begin{figure}[t!]
  \centerline{
    \mbox{\includegraphics[width=2.85in]{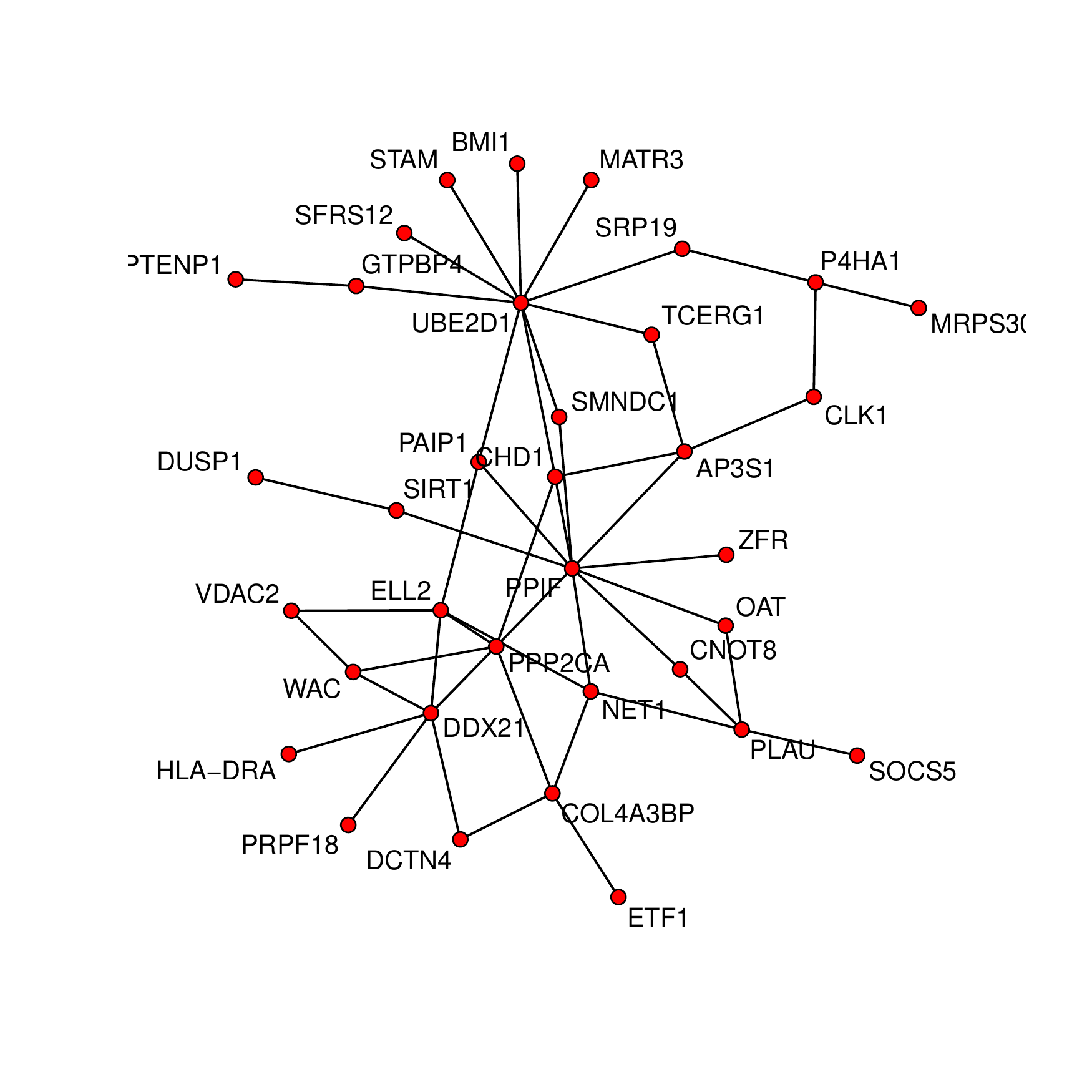}}
    \mbox{\includegraphics[width=2.85in]{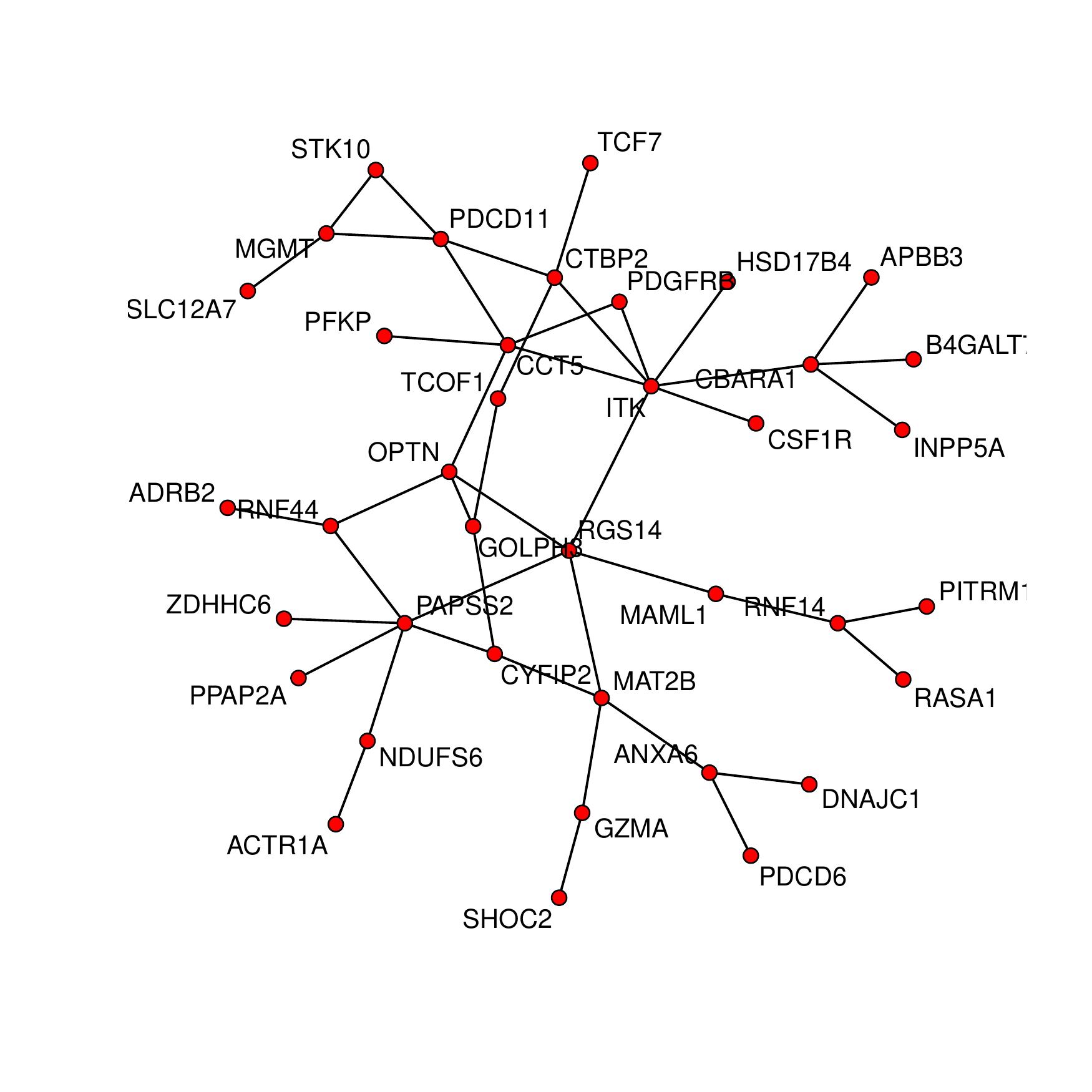}}
  }
  \caption{Graphs of the top $50$ marginal interactions in each cluster (and corresponding genes)}
  \label{fig:graphComp}
  \end{figure}

It appears, from this analysis, that there are two genetic pathways which are modified in Crohn's disease. Many of the genes in each cluster are already known to be indicated in Crohn's, but to our knowledge these interactions have not been considered.

\section{Dealing with Nuisance Variables}\label{sec:nuis}
Often, aside from the variables of interest, one may believe that other nuisance variables play a role in complex interactions. For example, it seems reasonable that many genes are conditionally independent given age, but are each highly correlated with age. Ignoring age, these genes would appear to be highly correlated, but this correlation is uninteresting to us. TMIcor can be adapted to deal with these nuisance variables provided there are few compared to the number of observations, they are continuous, and they are observed.

We resolve this issue by using partial correlations. Assume $x_j$ and $x_k$ are our variables of interest, and $z$ is a vector of potential confounders. Rather than comparing $\operatorname{cor}\left(x_j, x_k\right)$ in groups $1$ and $2$, we compare the partial correlations, $\operatorname{cor}\left(\left[x_j|z\right], \left[x_k|z\right]\right)$. This is done by first regressing our potential confounders, $Z$, out of all the other features, then running the remainder of the analysis as usual.

To adapt the original algorithm in Section~\ref{sec:method} to deal with nuisance variables we need only replace step $(1)$ by:
\begin{enumerate}
\item Replace our feature matrices $X_1$ and $X_2$ by
\[
\tilde{X}_m = \left[I - Z_m\left(Z_m^{\top}Z_m\right)Z_m^{\top}\right]X_m
\]
Now, mean center and scale $\tilde{X}$ within each group.
\end{enumerate}

We give more details motivating this approach and discussing potential computational advantages in appendix B.

\section{Asymptotics}\label{sec:asymptotics}
In this section we give two asymptotic results.  We show that if $n\rightarrow \infty$, and $\frac{\log p_n}{n} \rightarrow 0$, then under certain regularity conditions our procedure for testing marginal interactions (in the absence of nuisance variables) is asymptotically consistent --- with probability approaching $1$ it calls significant all true marginal interactions and makes no false rejections. Furthermore, using the permutation null, it also consistently estimates that the true FDR is converging to $0$. Because we only need $\frac{\log p_n}{n} \rightarrow 0$, $p_n$ may increase very rapidly in $n$.

We first give a result showing that for sub-Gaussian variables our null statistics converge to $0$ and our alternative statistics are asymptotically bounded away from $0$. The proof of this theorem is based on several technical lemmas which we relegate to appendix C.

\begin{theorem}\label{thm:con}

Let $\tilde{x}_{1(j)}$ and $\tilde{x}_{2(j)}$, $j=1,\ldots$ be random variables. Assume there is some $C>0$ such that for all $t\geq 0$
\[
\operatorname{P}\left(\left|x_{m(j)} - \operatorname{E}[x_{m(j)}]\right| > t\right) \leq \operatorname{exp}\left(1-t^2/C^2\right)
\]
for each $m=1,2$ . Let $\mu_{i(j)}$ denote the mean of $\tilde{x}_{m(j)}$ and $\sigma_{m(j)}^2$ its variance. For each $i\leq\infty$, let $x_{m(i,\cdot)}$ be independent realizations with the same distribution as $\tilde{x}_{m(\cdot)}$.

Let $p_n$ be a sequence of integers such that $\frac{\log p_n}{n} \rightarrow 0$.  Let $R_{m}$ be the correlation ``matrix'' (an infinite but countably indexed matrix) of the covariates from group $m$. Let $I$ denote the set of ordered pairs $(j,k)$ for which $R_{1(j,k)} \neq R_{2(j,k)}$, and $C_n$ denote the set of ordered pairs $(j,k)$ with $j,k\leq p_n$.

Assume for every $m$ and $j$, $\sigma_{m(j)}^2 \geq  \sigma_{min}^2$ (for some $\sigma_{min}^2 > 0$). Furthermore, assume that for all $(j,k)$ in each $I$, $\left|R_{1(j,k)} - R_{2(j,k)}\right| > \Delta_{\min}$ for some $\Delta_{\min}>0$ and that for $m=1,2$, $\operatorname{sup}_{j<k}\left|R_{m(j,k)}\right| < \rho_{\max}$ for some fixed $\rho_{\max} < 1$.\\

Now, given any $\epsilon_p >0$, and $0 < t < \Delta_{\min}$, if we choose $n$ sufficiently large, then with probability at least $1 - \epsilon_p$
\[
\left|T_{(j,k)}\right| \leq t
\]
for all $(j,k)$ in $C_n - I$, and
\[
\left|T_{(j,k)}\right| \geq t
\]
for all $(j,k)$ in $C_n\cap I$.
\end{theorem}

The notation here is a little bit tricky, but the result is very straightforward: under some simple conditions, we find all marginal interactions and make no false identifications.

While there were a number of assumptions in the above theorem, most of these are fairly trivial and will almost always be found in practice: the variance must be bounded away from $0$ and the correlations bounded away from $\pm 1$. The assumption that the correlation differences are bounded below by a fixed $\Delta_{\min}$ for true marginal interactions is a bit more cumbersome, but may easily be relaxed to $\Delta_{\min} \rightarrow 0$ at a slow enough rate that $\Delta_{min} /\left[\log p /n\right]^{1/2} \rightarrow \infty$.

The astute reader might note that our assumption bounding the variance away from $0$ seems strange --- the distribution of the sample correlation is independent of the variance. This is necessary only because we assumed the covariates have a subgaussian tail with a shared constant $C$. One could have relaxed the bounded variance assumption to the assumption that $\left\{x_j/\sigma_j\right\}_{j=1,\ldots}$ have a sub-Gaussian tail with a shared constant $C$.

\subsection{Permutation Consistency}

Now that we have shown our procedure has FDR converging to $0$, we would like to show that it asymptotically estimates FDR consistently as well. In particular we show that as $n\rightarrow \infty$, if $\frac{\log p}{n}\rightarrow 0$, then with probability approaching $1$, for a random permutation, our permuted statistics converge to $0$ uniformly in probability ($\max_{j,k}\left|T_{(j,k)}^*\right| \leq t$ for any fixed $t>0$ with probability converging to $1$). Thus our estimated FDR converges to $0$ under the same conditions as our true FDR.

We begin with some notation. Let us consider an arbitrary permutation of class labels, $\Pi$. Let $\hat{\pi}$ denote the proportion of observations from class $1$ that remain in class $1$ after permuting.

We discuss a somewhat simplified procedure in our proof, as otherwise the algebra becomes significantly more painful (without any added value in clarity), but it is straightforward to carry the proof through to the full procedure. In our original procedure, after permuting class labels we recenter and rescale our variables within each class. Because we already centered and scaled variables before permuting, this step will have very little effect on our procedure (though it does have the nice effect of never giving $|\rho^{*}| > 1$). In this proof we consider a procedure identical in every way except without recentering and rescaling within each permutation.

Before we give the theorem, we would like to define a few new terms for clarity. For a given permutation $\Pi$, let $\Pi_i(m)\in\left\{0,1\right\}$ be the permuted class of the $i$-th observation originally in class $m$. Furthermore, let $\Pi\left(m,l\right)$ be the set of observations in class $m$ that
are permuted to class $l$, and let $\Pi\left(\cdot,l\right)$ be the set of observations in both classes permuted to class $l$, ie.
\begin{align*}
\Pi\left(m,l\right) &= \left\{i:\,\Pi_i(m) = l\right\}\\
\Pi\left(\cdot,l\right) &= \left\{(i,m):\,\Pi_i(m) = l\right\}
\end{align*}

Now, we give a result which shows that for any fixed $t>0$ if our variables are sub-Gaussian with some other minor conditions, then for $n\rightarrow \infty$ and $\log p/n\rightarrow 0$ with probability approaching $1$, none of our permuted statistics will be larger than $t$, or in other words, as our true converged to $0$, so will our estimated FDR $0$. As before, the proof of this theorem is based on several technical lemmas which we again leave to appendix C.

\begin{theorem}\label{thm:perm}
Let $\tilde{x}_{1(j)}$ and $\tilde{x}_{2(j)}$, $j=1,\ldots$ be random variables with 
\[
\operatorname{P}\left(|x_{m(j)} - \operatorname{E}\left[x_{m(j)}\right]|\geq t\right) \leq 1 - e^{t^2/C}
\]
for all $t>0$, and each $m=1,2$, with some fixed $C>0$. Let $\mu_{m(j)}$ denote the mean of $\tilde{x}_{m(j)}$ and $\sigma_{m(j)}^2$ its variance. For each $i\leq\infty$, let $x_{m(i,\cdot)}$ be independent realizations with the same distribution as $\tilde{x}_{m(\cdot)}$.

Let $p_n$ be a sequence of integers such that $\frac{\log p_n}{n} \rightarrow 0$. Let $R_{m}$ be the correlation ``matrix'' (an infinite but countably indexed matrix) of the covariates from class $m$.

Assume for every $m,\,j$, $\sigma_{m(j)}^2 \geq  \sigma_{min}^2$ (for some $\sigma_{min}^2 > 0$). Furthermore, assume that for $m=1,2$, $\operatorname{sup}_{j<k}\left|R_{m(j,k)}\right| < \rho_{\max}$ for some fixed $\rho_{\max} < 1$.

Now, given any $\epsilon_p >0$, and $0 < t$, if we choose $n$ sufficiently large and let $\Pi$ be a random permutation, then with probability at least $1 - \epsilon_p$
\[
\left|T_{(j,k)}^*\right| \leq t
\]
for all $(j,k)$ with $j,k \leq p_n$ where 
\[
T_{(j,k)}^* = \operatorname{arctanh}\left(\hat{R}_{\operatorname{perm:1(j,k)}}\right) - \operatorname{arctanh}\left(\hat{R}_{\operatorname{perm:2(j,k)}}\right)
\]
and
\[
\hat{R}_{\textrm{perm}:m(j,k)} = \frac{1}{n}\sum_{(i,l)\in\Pi(\cdot,m)}\left(\frac{x_{l(i,j)} - \hat{\mu}_{l(j)}}{\hat{\sigma}_{l(j)}}\right)\left(\frac{x_{m(i,k)} - \hat{\mu}_{l(k)}}{\hat{\sigma}_{l(k)}}\right)
\]
\end{theorem}
The notation is again somewhat ugly, but the result is very straightforward: under some simple conditions, our permuted statistics are very small. In particular from the proof one can see that $\operatorname{sup}\left\{T_{(j,k)}^*\right\} = O_p\left(\sqrt{\log p_n/n}\right)$.

Note there is an implicit indexing of $n$ in $\hat{R}_{\textrm{perm}:m(j,k)}$ (it seemed unneccessary to add more indices). As in theorem~\ref{thm:con}, some of our conditions may be relaxed. Instead of bounding $\sigma_j^2$ below, we need only bound $C\sigma_j$ below. Also, rather than choose a fixed cutoff, $t>0$, we may use any sequence $\left\{t_n\right\}$ with $t_n/\left(\log p_n/n\right)^{1/2} \rightarrow \infty$. Also, as noted before, the result we have just shown ignores the restandardizing within each permutation, however it is straightforward (though algebraicly arduous, and not insightful) to extend this result to that case as well.

As a last note, in theorem~\ref{thm:perm}, we gave our consistency result for only a single permutation. This result can easily be extended to any fixed number of permutations using a union bound. This was left out of the original statement/proof as the notation is already clunky and the extension is straightforward.

Through theorems \ref{thm:con} and \ref{thm:perm} we have shown that, under fairly relaxed conditions, our procedure is asymptotically consistent at discovering marginal interactions and that the permutation null reflects this.

\section{Discussion}
In this paper we have discussed marginal interactions for logistic regression in the framework of forward and backward models. We have developed a permutation based method, TMIcor, which leverages the backward model. We have shown its efficacy on real and simulated data and given asymptotic results showing its consistency and convergence rate. We also plan to release a publically available {\tt R} implementation.

\section{Appendix A}\label{sec:perm}
In this section we give more details on our permutation-based estimate of FDR, and discuss a potential alternative. Recall that we are using the permutations to approximate
\begin{equation}\label{eq:numer}
\sum_{(j,k) \textrm{ null}} \operatorname{P}(|T_{(j,k)}| > t).
\end{equation}
For the moment, assume that all covariates in both classes have mean $0$
and variance $1$, and that we did not do any sample standarization. Then, under the null hypothesis that $R_{1(j,k)} =
R_{2(j,k)}$, $T_{(j,k)}$ calculated under the original class
assignments and $T^*_{(j,k)}$ calculated under any permuted class
assignments have the same distribution, so 
\[
\sum_{(j,k) \textrm{ null}} \operatorname{P}(|T_{(j,k)}| > t) = \sum_{(j,k) \textrm{ null}} \operatorname{P}(|T^*_{(j,k)}| > t)
\]
which is reasonably (and unbiasedly) approximated by
\[
\sum_{(j,k) \textrm{ null}} \frac{1}{A}\sum_{a=1}^AI(|T^{*a}_{(j,k)}| > t).
\]
Because we do not know which genes are null, our actual estimate of
\eqref{eq:numer} is 
\begin{align}\label{eq:bias}
\sum_{(j,k)} \frac{1}{A}\sum_{a=1}^AI(|T^{*a}_{(j,k)}| > t) &=\sum_{(j,k) \textrm{ null}} \frac{1}{A}\sum_{a=1}^A I(|T^{*a}_{(j,k)}| >
t)\\
&+ \sum_{(j,k) \textrm{ alternative}} \frac{1}{A}\sum_{a=1}^A I(|T^{*a}_{(j,k)}| > t)
\end{align}
This gives a slight conservative bias (especially small if most marginal interactions are null). One should also note that
unlike the null statistics, for the alternative $(j,k)$, $T^*_{(j,k)}$
are not distributed $N\left(0,\frac{2}{n-3}\right)$; they are still
mean $0$, but the variance is increased. However, this conservative bias is very slight --- in general there are few alternative hypotheses, and the variance increase is not large.

Because in practice we do not have mean $0$, variance $1$ for all
covariates in both classes, we must standardize before running our
procedure. Otherwise, instead of testing for a changing correlation,
we are actually testing for a different mean, variance, or correlation
between classes. The effect of standardizing with the sample mean and variance rather than the true values is asymptotically washed out, and while the variance of our tests is increased for small samples, this increase is only minimal.

An alternative to permutations, as discussed in \citet{efron2010ebayes}, is to directly
estimate the numerator using the approximate theoretical distribution of the null
statistics. Each null statistic is asymptotically
$N\left(0,\frac{1}{n_1-3} + \frac{1}{n_2-3}\right)$, so for $(j,k)$ null
\[
\operatorname{P}(|T_{(j,k)}| > t) = 2\Phi\left(-\frac{t (n_1 -3) (n_2-3)}{n_1
  + n_2 - 6}\right).
\]
Now we can conservatively approximate the quantity in Eq~\eqref{eq:numer} by
\begin{align*}
\sum_{(j,k) \textrm{ null}}P\left(|T_{(j,k)}| > t\right) &\leq p(p-1)/2 \cdot P\left(|T_{\textrm{null}}| > t\right)\\
&= p(p-1)\cdot\Phi\left(-\frac{t (n_1 -3) (n_2-3)}{n_1
  + n_2 - 6}\right)
\end{align*}
While this approach is reasonable and simple, it is less robust than using permutations, and in practice, even for truly Gaussian data, it is only slightly more efficient.

\section{Appendix B}
Before proceeding, we remind the reader that $x$ are our variables of interest and $z$ are potential confounding variables. Furthermore we are interested in comparing $\operatorname{cor}\left(\left[x_j|z\right], \left[x_k|z\right]\right)$ between groups. From basic properties of the Gaussian distribution we know that
\[
x|z\sim N\left[\mu_x + \Sigma_{(x,z)}\Sigma_{z}^{-1}\left(z - \mu_z\right), \Sigma_{(x|z)}\right]
\]
where $\Sigma_{(x|z)}$ is the variance/covariance matrix of $x$ given $z$, $\Sigma_{(x,z)}$ is the covariance matrix between $x$ and $z$, $\Sigma_z$ is the variance matrix of $z$, and $\mu_x$ and $\mu_z$ are the means of $x$ and $z$. Now, if $\mu_x,\,\mu_z,\,\Sigma_{(x,z)},$ and $\Sigma_{z}$ were known, then the MLE for $\Sigma_{(x|z)}$ would be
\[
\hat{\Sigma}_{(x|z)} = \frac{1}{n} \left[X - 1\mu_x^{\top} - \left(Z - 1\mu_z^{\top}\right)\Sigma_{z}^{-1}\Sigma_{(z,x)}\right]^{\top}\left[X - 1\mu_X^{\top} - \left(Z - 1\mu_Z^{\top}\right)\Sigma_{Z}^{-1}\Sigma_{(z,x)} \right].
\]
Unfortunately, these nuisance parameters are unknown. However we can also estimate them by maximum likelihood. This gives us the estimate
\begin{align*}
\hat{\Sigma}_{(X|Z)} &= \frac{1}{n}\left[\tilde{X} - \tilde{Z}\left(\tilde{Z}^{\top}\tilde{Z}\right)^{-1}\tilde{Z}^{\top}\tilde{X}\right]^{\top}\left[\tilde{X} - \tilde{Z}\left(\tilde{Z}^{\top}\tilde{Z}\right)^{-1}\tilde{Z}^{\top}\tilde{X}\right]\\
&=\frac{1}{n}\left[\operatorname{P}_{\tilde{Z}\perp}\left(\tilde{X}\right)\right]^{\top}\left[\operatorname{P}_{\tilde{Z}\perp}\left(\tilde{X}\right)\right]
\end{align*}
where $\tilde{Z}$ is the standardized version of $Z$, and $\tilde{X}$ is the standardized version of $X$, and $\operatorname{P}_{\tilde{Z}\perp}$ is the projection onto the orthogonal complement of the column space of $\tilde{Z}$. So, our estimate of partial correlation is just an estimate of correlation with $Z$ regressed out of both covariates. We use this to contruct our permutation null. In the orginal algorithm, we mean centered and scaled before permuting; here we do the equivalent ---  we project our variables of interest onto the orthogonal complement of our nuisance variables, and then center/scale them. Now we are ready to permute. We permute these ``residuals'', and calculate permuted correlations as before.

Before proceeding, we note that for $n$ sufficiently large $n$ ($n >> p$) one might use a similar approach to consider partial correlations rather than marginal correlations in our original algorithm (conditioning out all covariates except any particular $2$). However, in general $n << p$ and thus $\operatorname{P}_{\perp} \equiv 0$ rendering this approach ineffective --- this approach only works for nuisance variables because we assume that there are very few relative to the number of observations.

As stated in the text, to adapt the original algorithm to deal with nuisance variables we need only replace step $(1)$ by:
\begin{enumerate}
\item Replace our feature matrices $X_1$ and $X_2$ by
\[
\tilde{X}_m = \left[I - Z_m\left(Z_m^{\top}Z_m\right)Z_m^{\top}\right]X_m
\]
Now, mean center and scale $\tilde{X}$ within each group.
\end{enumerate}

One may note that we only calculate $\tilde{X}$ once per class, at the beginning of our procedure, not in each permutation. We do this for a similar reason that we standardize our variables before permuting --- because we are not testing the hypothesis that the relationship between $X$ and $Z$ is the same in both groups. If we relcalulate after each permutation then we are implicitly assuming that this relationship is the same in both groups under the null.

Even with nuisance variables this approach is very computationally fast. Projecting our original variables onto $Z\perp$ can be done in $O\left(npp_{\textrm{nuis}}\right)$ operations where $p_{\textrm{nuis}}$ is the number of nuisance variables. Thus the total runtime of this algorithm is $O\left(npp_{\textrm{nuis}} + Anp(p-1)/2\right)$ where $A$ is the number of permutations --- this is dominated by the second term, which is independent of the number of nuisance parameters. In contrast, if we were to use the standard approach (fitting pairwise logistic regressions with nuisance variables), its runtime would be $O\left[\left(iter\right)(3+p_{\textrm{nuis}})^2np(p-1)/2\right]$ where $iter$ is the number of iterations of the algorithm for finding the MLE. In general $A \sim 100$ and $iter \sim 5$. Now, since $(3+p_{\textrm{nuis}})^2$ grows very quickly in $p_{\textrm{nuis}}$, for even a small number of nuisance parameters the logistic approach becomes much slower.

\section{Appendix C}
This appendix contains the technical details from the theorems in section~$7$ of the main manuscript. We begin with a number of technical lemmas:

First, as one might imagine, if we can consistently estimate our correlation matrices, applying a Fisher transformation should not change much. We formalize this with the next lemma.

\begin{lemma}\label{lemma3}
Let $R_1$, $R_2$ be correlation matrices, and $\hat{R}_1$, $\hat{R}_2$ be estimates of $R_1$ and $R_2$.

Let $I$ be the set of ordered pairs $(j,k)$ where $R_{1(j,k)} \neq R_{2(j,k)}$. Assume for all $(j,k)$ in $I$, $\left|R_{1(j,k)} - R_{2(j,k)}\right| > \Delta_{\min}$ for some $\Delta_{\min} > 0$ and that for $m=1,2$ we have $\operatorname{sup}_{j<k}\left\|R_{m(j,k)}\right\|_{\infty} < \rho_{\max}$ for some fixed $\rho_{\max} < 1$.\\

Further assume that for $m=1,2$, $\left\|R_m - \hat{R}_m\right\|_{\infty} \leq \delta$ (for some $\delta < 1-\rho_{\max}$). Then for all $(j,k)$ in $I^{c}$ with $j \neq k$ we have
\begin{equation}\label{eq:close}
\left|\operatorname{arctanh}\left(\hat{R}_{1(j,k)}\right) - \operatorname{arctanh}\left(\hat{R}_{2(j,k)}\right)\right| \leq \frac{2\delta}{1-\left(\rho_{\max} + \delta\right)^2}
\end{equation}
and for all $(j,k)$ in $I$ with $j \neq k$ we have
\begin{equation}\label{eq:far}
\left|\operatorname{arctanh}\left(\hat{R}_{1(j,k)}\right) - \operatorname{arctanh}\left(\hat{R}_{2(j,k)}\right)\right| \geq \Delta_{\min} - 2\delta
\end{equation}
\end{lemma}
One immediate consequence of this lemma is that as $\delta \rightarrow 0$, for $(j,k)$ in $I^{C}$ our statistics $T_{(j,k)}$ converge to $0$ (at rate O($\delta)$), and for $(j,k)$ in $I$, $T_{(j,k)}$ are bounded away from $0$ (at a rate of at least O($\delta)$).

\begin{proof}[{\bf Proof of Lemma~\ref{lemma3}}]
We begin by showing that for all $(j,k)$ in $I^{c}$ with $j \neq k$ we have
\[
\left|\operatorname{arctanh}\left(\hat{R}_{1(j,k)}\right) - \operatorname{arctanh}\left(\hat{R}_{2(j,k)}\right)\right| \leq \frac{2\delta}{1-\left(\rho_{\max} + \delta\right)^2}
\]
The mean value theorem gives us that
\[
\left|\operatorname{arctanh}\left(\hat{R}_{1(j,k)}\right) - \operatorname{arctanh}\left(\hat{R}_{2(j,k)}\right)\right| \leq \operatorname{sup}_{r}\left|\frac{1}{1-r^2}\right|\left|\hat{R}_{1(j,k)} - \hat{R}_{2(j,k)}\right|
\]
where the supremum is taken over $r$ in $\left[\hat{R}_{1(j,k)},\, \hat{R}_{2(j,k)}\right]$. Note that for $m=1,2$, we have $|\hat{R}_{m(j,k)}| < \rho_{\max} + \delta$, and $\left|\hat{R}_{1(j,k)} - \hat{R}_{2(j,k)}\right| \leq 2\delta$, for $(j,k)$ not in $I$. Thus,
\[
\operatorname{sup}_{r}\left|\frac{1}{1-r^2}\right|\left|\hat{R}_{1(j,k)} - \hat{R}_{2(j,k)}\right| \leq \frac{2\delta}{1-\left(\rho_{\max} + \delta\right)^2}.
\]
Now for $(j,k)$ in $I$, we again use the mean value theorem:
\[
\left|\operatorname{arctanh}\left(\hat{R}_{1(j,k)}\right) - \operatorname{arctanh}\left(\hat{R}_{2(j,k)}\right)\right| \geq \operatorname{inf}_{r}\left|\frac{1}{1-r^2}\right|\left|\hat{R}_{1(j,k)} - \hat{R}_{2(j,k)}\right|
\]
and our result follows because $\left|\hat{R}_{1(j,k)} - \hat{R}_{2(j,k)}\right| \geq  \Delta_{\min} - 2\delta$.
\end{proof}


Now we consider convergence of these sample correlation matrices. We show that their convergence depends only on the convergence of the sample means ($\hat{\mu}_j$), variances ($\hat{\sigma}_j^2$), and pairwise inner products. We formalize this in the following lemma.

\begin{lemma}\label{lemma1}
Let $\tilde{x}_j$, $j=1,\ldots$ be random variables. Let $\mu_j$ denote the mean of $\tilde{x}_j$ and $\sigma_j^2$ its variance. Let $R_{j,k}$ be the correlation between $\tilde{x}_j$ and $\tilde{x}_k$. For each $i$, let $x_{i,\cdot}$ be independent realizations with the same distribution as $\tilde{x}_{\cdot}$ (eg. $x_{i,j}$ has the marginal distribution of $\tilde{x}_j$).

For any given $\epsilon > 0$, there exists $\delta > 0$ such that if
\begin{equation}\label{eq:bnd}
\operatorname{sup} \left\{\left|\hat{\sigma}_j - \sigma_j\right|,\, \left|\hat{\mu}_j - \mu_j\right|,\,\left|\frac{(1/n)\sum_{i\leq n}x_{i,j}x_{i,k}}{\sigma_j\sigma_k} - \frac{\mu_j\mu_k}{\sigma_j\sigma_k} - R_{j,k}\right|\right\}_{j,k} \leq \delta
\end{equation}
then
\begin{equation}\label{eq:bnd0}
\operatorname{sup}_{j<k \leq p} \left|\hat{R}_{j,k} - R_{j,k}\right| \leq \epsilon
\end{equation}
Furthermore, one can choose $\delta = O(\epsilon)$
\end{lemma}

\begin{proof}[{\bf Proof of Lemma~\ref{lemma1}}]
We begin by noting that the distribution of $\hat{R}_{j,k}$ is independent of $\mu_j$, $\mu_k$, $\sigma_j$ and $\sigma_k$. For ease of notation we assume $\mu_j = \mu_k = 0$ and $\sigma_j = \sigma_k = 1$.\\

To see that \eqref{eq:bnd} is sufficient for \eqref{eq:bnd0} we write $\hat{R}_{j,k} - R_{j,k}$ as
\begin{align*}
\left|\hat{R}_{j,k} - R_{j,k}\right| &= \left|\frac{\left(1/n\right)\sum_{i=1}^n x_{i,j}x_{i,k}}{\hat{\sigma}_j\hat{\sigma}_k} - \frac{\hat{\mu}_j\hat{\mu}_k}{\hat{\sigma}_j\hat{\sigma}_k} - R_{j,k}\right|\\
&\leq \left|\frac{1}{n}\sum_{i=1}^n x_{i,j}x_{i,k}\right|\left|\left(\frac{1}{\hat{\sigma}_j\hat{\sigma}_k} - 1\right)\right|\\
&+ \left|\frac{1}{n}\sum_{i=1}^n x_{i,j}x_{i,k} - R_{j,k}\right| + \left|\frac{\hat{\mu}_j\hat{\mu}_k}{\hat{\sigma}_j\hat{\sigma}_k}\right|
\end{align*}
We first note that $\left|\frac{1}{n}\sum_{i=1}^n x_{i,j}x_{i,k} - R_{j,k}\right| < \delta$. Thus we need only consider $\left|\frac{\hat{\mu}_j\hat{\mu}_k}{\hat{\sigma}_j\hat{\sigma}_k}\right|$ and $\left|\left(\frac{1}{\hat{\sigma}_j\hat{\sigma}_k} - 1\right)\right|$. Expanding these terms using the fact that $1/(1-\delta) = 1 + O(\delta)$, it is straightforward to see that the whole expression converges to $0$ at rate $O(\delta)$. This completes our proof.
\end{proof}

Now that we have reduced convergence to that of the sample mean, variance, and inner products, we show particular circumstances under which our estimation is consistent, and give rates of convergence.

\begin{lemma}\label{lemma2}
Let $\tilde{x}_j$, $j=1,\ldots$ be random variables. Assume there is some $C>0$ such that for all $t\geq 0$
\[
\operatorname{P}\left(\left|x_j - \operatorname{E}[x_j]\right| > t\right) \leq \operatorname{exp}\left(1-t^2/C^2\right)
\]
(These are known as sub-Gaussian random variables). Let $\mu_j$ denote the mean of $\tilde{x}_j$ and $\sigma_j^2$ its variance. Let $R_{j,k}$ be the correlation between $\tilde{x}_j$ and $\tilde{x}_k$. For each $i$, let $x_{i,\cdot}$ be independent realizations with the same distribution as $\tilde{x}$.

Let $\delta,\, \epsilon_p > 0$ be given. Then for $n$ sufficiently large and $\frac{\log p}{n}$ sufficiently small we have that
\begin{equation}\label{eq:lem2}
\operatorname{sup} \left\{\left|\hat{\sigma}_j - \sigma_j\right|,\, \left|\hat{\mu}_j - \mu_j\right|,\,\left|\frac{(1/n)\sum_{i\leq n}x_{i,j}x_{i,k}}{\sigma_j\sigma_k} - \frac{\mu_j\mu_k}{\sigma_j\sigma_k} - R_{j,k}\right|\right\}_{j,k\leq p} \leq \delta
\end{equation}
with probability greater than $1-\epsilon_p$. In particular one can choose $\delta = O\left(\log p / n\right)^{1/2}$.
\end{lemma}
The class of subgaussian random variables is rather broad, containing gaussian random variables and all bounded random variables. Applying this lemma, we are able to show consistency for the wide class of variables with sufficiently light tails.

In the proof of this lemma we get a convergence rate of $\delta = O\left(\log p / n\right)^{1/2}$. This rate agrees with the literature for other similar problems in covariance estimation (\citet{bickel2008} among others).

\begin{proof}[{\bf Proof of Lemma~\ref{lemma2}}]
We will begin by bounding $\left|\hat{\mu}_j - \mu_j\right|$. If we consider Lemma~$5.10$ of \citet{vershynin2010} we see that
\[
\operatorname{P}\left(\left|\hat{\mu}_j - \mu_j\right| > t\right) \leq e\cdot \operatorname{exp}\left[-\left(\tilde{C}t^2\right)n\right]
\]
where $\tilde{C}$ is some function of $C$ (one can prove this Hoeffding type inequality by an exponential Markov argument). Applying the union bound to this we see that
\[
\operatorname{P}\left(\operatorname{sup}_{j\leq p}\left|\hat{\mu}_j - \mu_i\right| > t\right) \leq 3 p \operatorname{exp}\left[-\left(\tilde{C}t^2\right)n\right]
\]
If we set $t = \left(\sqrt{1/C}\right)\sqrt{\frac{q + \log p}{n}}$ then we have
\[
\operatorname{P}\left(\operatorname{sup}_{j\leq p}\left|\hat{\mu}_j - \mu_j\right| > t\right) \leq  e^{1-q},
\]
bounding $\left|\hat{\mu}_j - \mu_j\right|$.\\
Next we bound $\left|\hat{\sigma}_j - \sigma_i\right|$. We first note that
\[
\left|\hat{\sigma}_j - \sigma_j\right| = \frac{\left|\hat{\sigma}_j^2 - \sigma_j^2\right|}{\hat{\sigma}_j + \sigma_j} \leq \frac{\left|\hat{\sigma}_j^2 - \sigma_j^2\right|}{\sigma_j}
\]
because $\hat{\sigma_j}, \sigma_j >0$.
so we need only consider convergence of $\hat{\sigma}_j^2 - \sigma_j^2$. Next note that
\[
\frac{1}{n}\sum_i \left(x_{i,j} - \bar{x}_j\right)^2 - \frac{1}{n}\sum_i \left(x_{i,j} - \mu_j\right)^2 = -\left(\bar{x}_j - \mu_j\right)^2
\]
So now if we can bound $\left|\frac{1}{n}\sum_i \left(x_{i,j} - \mu_j\right)^2 - \sigma_j^2\right|$ and $\left(\bar{x}_j - \mu_j\right)^2$, then we can bound $|\hat{\sigma}_j^2 - \sigma_j^2|$.\\

To bound $\left|\frac{1}{n}\sum_i \left(x_{i,j} - \mu_j\right)^2 - \sigma_j^2\right|$, we first note that if $x_{i,j}$ is sub-Gaussian then $(x_{i,j} - \mu_j)^2$ is subexponential; ie
\[
\operatorname{P}\left(\left(x_{i,j} - \mu_j\right)^2 - \sigma_i > t\right) \leq \operatorname{exp}\left(-C_1 t\right)
\]
for some fixed $C_1$. Now we apply Corollary~$5.17$ of \citet{vershynin2010}, and get that for any $t$ sufficiently small (independent of $n$)
\[
\operatorname{P}\left(\frac{1}{n}\sum_i\left(x_{i,j} - \mu_j\right)^2 > t\right) \leq 2\operatorname{exp}\left(-\tilde{C}_1 t^2\right)
\]
for some fixed $\tilde{C}_1$. Bounding $\left(\bar{x}_j - \mu_j\right)^2$ is also quite straightforward (we just use the bound for $\left|\bar{x}_j - \mu_j\right|$)
\[
P\left(\left(\bar{x}_j - \mu_j\right)^2 \geq t\right) \leq e \operatorname{exp}\left[-\left(\tilde{C}t\right)n\right]
\]
We note that for $t<1$, $t^2 < t$. Let $\bar{C} = \min\{\tilde{C}_1,\tilde{C}\}$. Now, combining these inequalities with the triangle inequality we have
\begin{align*}
P\left(\left|\hat{\sigma}_j^2 - \sigma_j^2\right| \geq t\right) &\leq e\operatorname{exp}\left[-\left(\tilde{C}t\right)n\right] + 2\operatorname{exp}\left(-\tilde{C}_1 t^2\right)\\
& \leq 5\operatorname{exp}\left[-\bar{C}t^2n\right]
\end{align*}
for $t$ sufficiently small. Now finally,
\[
P\left(\left|\hat{\sigma}_j - \sigma_j\right| \geq t\right) \leq P\left(\left|\hat{\sigma}_j^2 - \sigma_j^2\right| \geq t\sigma_{\min}\right) \leq 5\operatorname{exp}\left[-\bar{C}\sigma_{\min}^2t^2n\right].
\]

Using the union bound again, we get
\[
P\left(\operatorname{sup}_{j}\left|\hat{\sigma}_j^2 - \sigma_j^2\right| \geq t\right) \leq 5p\operatorname{exp}\left[-\bar{C}t^2n\right].
\]
so
\[
P\left(\operatorname{sup}_{j}\left|\hat{\sigma}_j - \sigma_j\right| \geq t\right) \leq 5p\operatorname{exp}\left[-\bar{C}\sigma_{\min}^2t^2n\right].
\]

Finally, we need to bound $\left|\frac{(1/n)\sum_{i\leq n}x_{i,j}x_{i,k}}{\sigma_j\sigma_k} - \frac{\mu_j\mu_k}{\sigma_j\sigma_k} - \rho_{j,k}\right|$. This is slightly trickier but still not terrible. We first note that
\[
(1/n)\sum_{i\leq n}x_{i,j}x_{i,k} - \mu_j\mu_k = (1/n)\sum_{i\leq n}\left(x_{i,j} - \mu_j\right)\left(x_{i,k}-\mu_k\right)
\]
We also see that
\begin{align*}
2\sum_{i\leq n}\left(x_{i,j} - \mu_j\right)\left(x_{i,k}-\mu_k\right) &= \sum_{i\leq n}\left[\left (x_{i,j} - \mu_j\right) + \left(x_{i,k}-\mu_k\right)\right]^2\\
& -  \sum_{i\leq n}\left (x_{i,j} - \mu_j\right)^2 -  \sum_{i\leq n}\left (x_{i,k} - \mu_k\right)^2
\end{align*}
Now to bound the above quantity we consider the moment generating function of $x_{i,j} - \mu_j + x_{i,k}-\mu_k$. This not necessarily the sum of independent random variables, still by Cauchy Schwartz we have
\begin{align*}
&\operatorname{E}\left[\operatorname{exp}\left[t\left(x_{i,j} - \mu_j + x_{i,k}-\mu_k\right)\right]\right]\\
&\leq \operatorname{max}\left\{\operatorname{E}\left[\operatorname{exp}\left[2t\left(x_{i,j} - \mu_j\right)\right]\right],\operatorname{E}\left[\operatorname{exp}\left[2t\left(x_{i,k} - \mu_k\right)\right]\right]\right\}
\end{align*}
It is a well known fact that sub-gaussan random variables can be charaterized by their MGF (shown in \citet{vershynin2010}), and this is still the moment generating function of a subgaussian random variable. Thus, $\left(x_{i,j} - \mu_j + x_{i,k}-\mu_k\right)^2$ is sub-exponential, and again by Corollary~$5.17$ of \citet{vershynin2010} we have that
\begin{align*}
&\operatorname{P}\left(\left|\frac{1}{n}\sum_{i}\left(x_{i,j} - \mu_j + x_{i,k}-\mu_k\right)^2 - \sigma_j^2 - \sigma_k^2 - 2\sigma_j\sigma_k\rho_{j,k}\right| > t\right)\\
&\leq 2\operatorname{exp}\left[-C_2t^2n\right].
\end{align*}
for $t>0$ sufficiently small and some fixed $C_2 > 0$. Now, stringing all of these together with the triangle inequality we have that
\begin{align*}
&\operatorname{P}\left(\left|\frac{2}{n}\sum_{i\leq n}\left(x_{i,j} - \mu_j\right)\left(x_{i,k}-\mu_k\right) - 2\rho\sigma_j\sigma_k\right| > 3t\right)\\
&\leq \operatorname{P}\left(\left|\frac{1}{n}\sum_{i\leq n}\left(x_{i,j} - \mu_j + x_{i,k}-\mu_k\right)^2 - \sigma_j^2 + \sigma_k^2 - 2\sigma_j\sigma_k\rho_{j,k}\right| > t\right)\\
&+ \operatorname{P}\left(\left|\frac{1}{n}\sum_{i\leq n}\left (x_{i,j} - \mu_j\right)^2 - \sigma_j^2\right| > t\right) + \operatorname{P}\left(\left|\frac{1}{n}\sum_{i\leq n}\left (x_{i,k} - \mu_k\right)^2 - \sigma_k^2\right| > t\right)\\
&\leq 2\operatorname{exp}\left[-C_2t^2n\right] +  2*5\operatorname{exp}\left[-\bar{C}t^2n\right]\\
&\leq 12\operatorname{exp}\left[-\bar{C}_1t^2n\right]
\end{align*}
for all $t>0$ sufficiently small with some fixed $\bar{C}_1>0$. Taking this a step further, and applying the union bound, we see that
\[
P\left(\operatorname{sup}_{j,k}\left|\frac{(1/n)\sum_{i\leq n}x_{i,j}x_{i,k}}{\sigma_j\sigma_k} - \frac{\mu_j\mu_k}{\sigma_j\sigma_k} - \rho_{j,k}\right| > t\right) \leq 12p^2 \operatorname{exp}\left[-\bar{C}_2t^2n\right]
\]
for some fixed $\bar{C}_2$.\\

Now that we have bounded each term, we see that \eqref{eq:lem2} happens with probability at most
\begin{align*}
&12p^2 \operatorname{exp}\left[-\bar{C}_2\delta^2n\right] + 2*5p\operatorname{exp}\left[-\bar{C}\sigma_{\min}^2\delta^2n\right] + 2*3p\operatorname{exp}\left[-\tilde{C}\delta^2n\right]\\
&\leq 28p^2\operatorname{exp}\left[-\mathbf{C}\delta^2 n\right]
\end{align*}


for $\delta$ sufficiently small where $\mathbf{C} = \min\left\{\bar{C}\sigma_{\min}^2,\bar{C_2},\tilde{C}\right\}$. Thus, if $\delta = \left(\frac{q + 2\log p}{\mathbf{C}n}\right)^{1/2}$ then we have \eqref{eq:lem2} with probability at least $1-28e^{-q}$. If $n$ is sufficiently large, and $\frac{\log p}{n}$ sufficiently small, then for any $q$, $\delta$ can be made arbitrarily small.
\end{proof}

Now, we combine these lemmas to show that under certain conditions, for a given cutoff $t$, as $n\rightarrow\infty$ if $\log p/n\rightarrow 0$ then, with probability approaching $1$, all true marginal interactions have $|T_{i,j}| > t$, and all null statistics will have $|T_{i,j}| < t$ (ie. we asymptotically find all true interactions and make no false rejections).

Before we begin, it deserves mention that we use slightly different notation than in the discussion of our algorithm in Section~$3$. Rather than having $X_{i,\cdot}$ denote the $i$-th observation overall, and letting $y(i)$ denote its group (where $i$ ranged from $1$ to the total number of observations in both groups), we split up our observations by group, letting $x_{m(i,\cdot)}$ denote the $i$-th observation from group $m$ (now $i$ ranges from $1$ to the total number of observations in group $m$). This change simplifies notation in the statement of the theorem and its proof. We also assume equal group sizes ($n_1 = n_2 = n$), this again simplifies notation but can be relaxed to $n_1/(n_1 + n_2) \rightarrow \alpha \in (0,1)$. 

\begin{proof}[{\bf Proof of Theorem~6.1}]
This result is a straightforward corollary of our $3$ lemmas:

First choose an arbitrary $\epsilon_p > 0$, and $0 < t < \Delta_{\min}$. If we consider Lemma~\ref{lemma1}, we see that the conclusion of our theorem holds if we can find a bound on the sup-norm distance between each correlation matrix and its MLE (a bound I will call $\delta_1$) which satisfies
\[
\max\left\{\frac{2\delta_1}{1-\left(\rho_{\max} + \delta_1\right)^2},\, \Delta_{\min} - 2\delta_1\right\} \leq t.
\]
Because $\rho_{\max} < 1$, $\delta_1 > 0$ sufficiently small will satisfy this.\\

Now applying Lemma~\ref{lemma2}: if we choose $\delta_2$ sufficiently small (but still of $O(\delta_1)$), then if
\begin{equation}\label{thm:bound1}
\operatorname{sup} \left\{\left|\hat{\sigma}_j - \sigma_j\right|,\, \left|\hat{\mu}_j - \mu_j\right|,\,\left|\frac{(1/n)\sum_{i\leq n}x_{i,j}x_{i,k}}{\sigma_j\sigma_k} - \frac{\mu_j\mu_k}{\sigma_j\sigma_k} - \rho_{j,k}\right|\right\}_{j,k} \leq \delta_2
\end{equation}
we have that the sup norm distance between each correlation matrix and its MLE is bounded by $\delta_1$: for $m=1,2$
\[
\left\| \hat{R}_m - R_m \right\|_{\infty}\leq \delta_1
\]

Finally, by Lemma!\ref{lemma3}, we see that if $n$ is sufficiently large and $\log p / n$ is sufficiently small then \eqref{thm:bound1} holds with probability at least $1-\epsilon_p$. This finishes our proof.
\end{proof}

\subsection{Proofs of Permutation Results}
To begin, we prove a Lemma which does most of the leg-work for our eventual theorem. It says that for a reasonably balanced permutation, for $n$ sufficiently large and $\log p/n$ sufficiently small, both of our permuted sample correlation matrices will be very close to the average of the $2$ population correlation matrices.

\begin{lemma}\label{lemma:perm}
Let $\tilde{x}_{1(j)}$ and $\tilde{x}_{2(j)}$, $j=1,\ldots$ be random variables with 
\[
\operatorname{P}\left(|x_{m(j)} - \operatorname{E}\left[x_{m(j)}\right]|\geq t\right) \leq 1 - e^{t^2/C}
\]
for all $t>0$, and each $m=1,2$, with some fixed $C>0$. Let $\mu_{m(j)}$ denote the mean of $\tilde{x}_{m(j)}$ and $\sigma_{m(j)}^2$ its variance. For each $i < \infty$, let $x_{m(i,\cdot)}$ be independent realizations with the same distribution as $\tilde{x}_{m(\cdot)}$.

Let $p_n$ be a sequence of integers such that $\frac{\log p_n}{n} \rightarrow 0$. Let $R_{m}$ be the correlation ``matrix'' (an infinite but countably indexed matrix) of the covariates from class $m$. Define $R_{\operatorname{perm}}$ to be the average of the two,
\[
R_{\textrm{perm}} = \frac{1}{2} R_1 + \frac{1}{2}R_2
\]
Let $\hat{\mu}_{m(j)}$ and $\hat{\sigma}_{m(j)}^2$ be the pre-permuted estimates of the mean and variance (in each class):
\[
\hat{\mu}_{m(j)} = \frac{1}{n}\sum_{i \leq n} x_{m(i,j)}
\]
and
\[
\hat{\sigma}_{m(j)}^2 = \frac{1}{n}\sum_{i \leq n} \left(x_{m(i,j)} - \hat{\mu}_{m(j)}\right)^2.
\]
Further, define
\[
\hat{R}_{\textrm{perm}:m(j,k)} = \frac{1}{n}\sum_{(i,l)\in\Pi(\cdot,m)}\left(\frac{x_{m(i,j)} - \hat{\mu}_{m(j))}}{\hat{\sigma}_{m(j)}}\right)\left(\frac{x_{m(i,k)} - \hat{\mu}_{m(k)}}{\hat{\sigma}_{m(k)}}\right)
\]
our permuted correlation between covariates $j$ and $k$ in class $m$.

Assume for every $j$, $\sigma_j^2 \geq  \sigma_{min}^2 > 0$. Now for any $\epsilon >0$, $\delta>0$, one can find $n$ sufficiently large such that for any permutation, $\Pi$ with 
\[
\left|\hat{\pi} - \frac{1}{2}\right|\leq \frac{\delta}{12}
\]
(where $\hat{\pi}$ is the proportion of class $1$ that remains fixed under $\Pi$). We have
\begin{equation}\label{perm:bnd}
\left\|R_{\textrm{perm}} - \hat{R}_{\textrm{perm}:m}\right\|_{\infty} \leq \delta
\end{equation}
for both $m=1,2$ with probability at least $1-\epsilon$.
\end{lemma}

\begin{proof}[{\bf Proof of Lemma~\ref{lemma:perm}}]
We first consider only $m=1$. If we can show that
\begin{equation}\label{perm:bnd}
\left\|R_{\textrm{perm}} - \hat{R}_{\textrm{perm}:m}\right\|_{\infty} \leq \delta
\end{equation}
with high probability for $m=1$, then by symmetry we have it for $m=2$, and by a simple union bound we have it for both simultaneously.\\

Now, we begin by decomposing our sample permuted correlation matrix
\begin{align*}
\hat{R}_{\textrm{perm}:1} &= \frac{1}{n}\sum_{(i,m)\in\Pi(\cdot,1)}\left(\frac{x_{m(i,j)} - \hat{\mu}_{m(j)}}{\hat{\sigma}_{m(j)}}\right)\left(\frac{x_{m(i,k)} - \hat{\mu}_{m(k)}}{\hat{\sigma}_{m(k)}}\right)\\
&= \hat{\pi}\hat{R}_{\textrm{perm}:1}^{(1)} + \left(1 - \hat{\pi}\right)\hat{R}_{\textrm{perm}:1}^{(2)}
\end{align*}
where $\hat{R}_{\textrm{perm}:1}^{(l)}$ is a matrix defined by
\begin{equation}\label{eq:contrib}
\hat{R}_{\textrm{perm}:1(j,k)}^{(l)} = \frac{1}{\tilde{n}_l}\sum_{i\in\Pi(l,1)}\left(\frac{x_{l(i,j)} - \hat{\mu}_{1(j)}}{\hat{\sigma}_{1(j)}}\right)\left(\frac{x_{l(i,k)} - \hat{\mu}_{1(k)}}{\hat{\sigma}_{1(k)}}\right)
\end{equation}
where $\tilde{n}_l$ is the number of elements from group $l$ permuted to group $1$ (ie. the cardinality of $\Pi(l,l)$ or more explicitly $\tilde{n}_1 = \hat{\pi}n$ and $\tilde{n}_2 = (1-\hat{\pi}n$). The quantity \eqref{eq:contrib} is just the contribution from observations originally in class $l$ to the permuted correlation matrix for class $1$. Thus by the triangle inequality
\begin{align} \label{eq:permTriangle}
\left\|R_{\textrm{perm}} - \hat{R}_{\textrm{perm}:1}\right\|_{\infty} &\leq \left\|\frac{1}{2} R_1 - \hat{\pi}\hat{R}_{\textrm{perm}:1}^{(1)}\right\|_{\infty} + \left\|\frac{1}{2} R_2 - \left(1 - \hat{\pi}\right)\hat{R}_{\textrm{perm}:1}^{(1)}\right\|_{\infty}\\
&\leq \frac{1}{2} \left\|R_1 - \hat{R}_{\textrm{perm}:1}^{(1)}\right\|_{\infty} + \frac{1}{2}\left\|R_2 - \hat{R}_{\textrm{perm}:1}^{(2)}\right\|_{\infty}\notag\\
&+ \left|\hat{\pi} - \frac{1}{2}\right|\left(\left\|\hat{R}_{\textrm{perm}:1}^{(1)}\right\|_{\infty} + \left\|\hat{R}_{\textrm{perm}:1}^{(2)}\right\|_{\infty}\right)\notag
\end{align}
If we consider $\hat{R}_{\textrm{perm}:1}^{(1)}$, we see that it is essentially a sample correlation matrix (using only the $\hat{\pi}n$ observations that were fixed in class $1$ by $\Pi$ for the inner product). We can make a similar observation for $\hat{R}_{\textrm{perm}:1}^{(2)}$. Now, for $n$ sufficiently large, because $|\frac{1}{2} - \hat{\pi}|$ is small, we can make $\hat{\pi}n$ and $\left(1-\hat{\pi}\right)n$ as large as we would like. Thus, by a combination of Lemma~\ref{lemma3} and Lemma~\ref{lemma2}, we have that 
\[
\left\|R_l - \hat{R}_{\textrm{perm}:1}^{(l)}\right\|_{\infty} < \delta/3
\]
with probability greater than $1-\epsilon/3$. Furthermore, using the same Lemmas we get
\[
\left\|\hat{R}_{\textrm{perm}:1}^{(1)}\right\|_{\infty} + \left\|\hat{R}_{\textrm{perm}:1}^{(2)}\right\|_{\infty} \leq 4
\]
with probability at least $1-\epsilon/3$ (this bound can easily be made tighter, and if we were to standardize within permutation this bound is trivial). Plugging this in with the assumed bound on $\left|\hat{\pi} - \frac{1}{2}\right|$ completes the proof.
\end{proof}

Now, we use this Lemma (along with some of our previous Lemmas) to show that for any fixed $t>0$ if our variables are subgaussian with some other minor conditions, then for $n\rightarrow \infty$ and $\log p/n\rightarrow 0$ with probability approaching $1$, none of our permuted statistics will be larger than $t$, or in other words our estimated FDR will converge to $0$.

\begin{proof}[{\bf Proof of Theorem~6.2}]
First we choose an arbitrary $2\epsilon_p >0$ and $t>0$. If we consider Lemma~$7.1$, we see that if we find some $\delta >0$ satisfying
\begin{equation}\label{eq:need}
\left\|R_{\textrm{perm}} - \hat{R}_{\textrm{perm}:m}\right\|_{\infty}
\end{equation}
for $m=1,2$ with probability at least $1-\epsilon_p$ and
\begin{equation}\label{eq:bd}
\frac{2\delta}{1-(\rho_{\max} + \delta)^2} \leq t
\end{equation}
then we have satisfied our claim. Because, $\rho_{\max} < 1$, there exists some $\delta >0$ satisfying \eqref{eq:bd}. Now, we first note that, for $n$ sufficiently large, standard concentration inequalities give us that
\[
\left|\hat{\pi} - \frac{1}{2}\right| \leq \delta/12
\]
with probability greater than $1 - \epsilon_p$. If we apply Lemma~$7.5$ with this bound on $\hat{\pi}$ and combine the probabilities with the union bound, we get that for $n$ sufficiently large \eqref{eq:need} is violated with at most probability $2\epsilon_p$. This completes our proof.
\end{proof}

\section{Acknowledgments}
We would like to thank Jonathan Taylor and Trevor Hastie for their helpful comments and insight.

\bibliographystyle{abbrvnat}
\bibliography{/home/nsimon/texlib/simon}

\end{document}